\documentclass{article}
\usepackage[in]{fullpage}
\usepackage{graphicx,psfrag,amsmath,amsthm,amsfonts,amssymb,verbatim,algpseudocode,comment,mathtools}
\usepackage{bbm}
\usepackage[linesnumbered,ruled,vlined]{algorithm2e}
\usepackage[numbers]{natbib} 
\usepackage{graphicx}
\usepackage{xcolor}
\usepackage{enumerate}
\usepackage{textcomp}
\usepackage{caption}
\usepackage{subcaption}
\usepackage{mathrsfs}
\usepackage[citecolor=blue]{hyperref} 
\usepackage{thmtools,thm-restate}
\usepackage[makeroom]{cancel}
\usepackage{dirtytalk}

\newtheorem{theorem}{Theorem}
\newtheorem{lemma}{Lemma}

\newtheorem{remark}{Remark}
\newtheorem{proposition}{Proposition}

\theoremstyle{definition}
\newtheorem{definition}{Definition}
\newtheorem{edefinition}{Existing Definition}

\newtheorem{example}{Example}

\usepackage{epigraph}
\usepackage{etoolbox}

\def\actions{\mathcal{A}}

\def\histories{\mathcal{H}}

\def\E{\mathbb{E}}

\def\H{\mathbb{H}}

\def\I{\mathbb{I}}
\def\Pr{\mathbb{P}}
\def\R{\mathbb{R}}
\def\1{\mathbf{1}}
\def\0{\mathbf{0}}

\newcommand\numberthis{\addtocounter{equation}{1}\tag{\theequation}}

\usepackage[mathscr]{euscript}
\hypersetup{
    colorlinks=true,
    linkcolor=blue,
    filecolor=magenta,      
    urlcolor=purple,
}
\usepackage{xcolor}

\newcount\Comments
\newcommand{\kibitz}[2]{\ifnum\Comments=1{\textcolor{#1}{\textsf{\footnotesize #2}}}\fi}
\usepackage{color}
\definecolor{darkred}{rgb}{0.7,0,0}
\definecolor{darkgreen}{rgb}{0.0,0.5,0.0}
\definecolor{darkblue}{rgb}{0.0,0.0,0.5}
\definecolor{teal}{rgb}{0.0,0.5,0.5}
\definecolor{dogwoodrose}{rgb}{0.84, 0.09, 0.41}
\definecolor{electriccrimson}{rgb}{1.0, 0.0, 0.25}
\definecolor{folly}{rgb}{1.0, 0.0, 0.31}
\definecolor{frenchrose}{rgb}{0.96, 0.29, 0.54}
\definecolor{cadetgrey}{rgb}{0.57, 0.64, 0.69}

\allowdisplaybreaks[4]

\usepackage{setspace}
\usepackage{authblk}
\title{A Definition of Non-Stationary Bandits}
\author[1]{Yueyang Liu} 
\author[2 \footnote{Xu Kuang published under a different full name in earlier versions of this manuscript. Please use “Y.
Liu, X. Kuang and B. Van Roy” when citing this paper.}]{Xu Kuang} 
\author[1,3]{Benjamin Van Roy} 
\affil[1]{Department of Management Science and Engineering, Stanford University}
\affil[2]{Stanford Graduate School of Business}
\affil[3]{Department of Electrical Engineering, Stanford University}
\date{}                     
\begin{document}

\maketitle

\begin{abstract}
Despite the subject of non-stationary bandit learning having attracted much recent attention, we have yet to identify a formal definition of non-stationarity that can consistently distinguish non-stationary bandits from stationary ones. 
Prior work has characterized non-stationary bandits as bandits for which the reward distribution changes over time. We demonstrate that this definition can ambiguously classify the same bandit as both stationary and non-stationary; 
this ambiguity arises in the existing definition's dependence on the latent sequence of reward distributions. 
Moreover, the definition has given rise to two widely used notions of regret: the dynamic regret and the weak regret. These notions are not indicative of qualitative agent performance in some bandits. 
Additionally, this definition of non-stationary bandits has led to the design of agents that explore excessively. We introduce a formal definition of non-stationary bandits that resolves these issues. 
Our new definition provides a unified approach, applicable seamlessly to both Bayesian and frequentist formulations of bandits. 
Furthermore, our definition ensures consistent classification of two bandits offering agents indistinguishable experiences, categorizing them as either both stationary or both non-stationary. This advancement provides a more robust framework for non-stationary bandit learning.
\end{abstract} 

\section{Introduction}\label{sec:intro}
The subject of non-stationary bandit learning has attracted much recent attention. However, what constitutes a non-stationary bandit as opposed to a stationary one has yet to be formally defined. Non-stationary bandits have typically been defined as bandits for which the \say{reward distribution} changes 
over time. 
For example, in the seminal paper by Whittle \citep{whittle1988restless}, non-stationary bandits are referred to as \emph{restless bandits}, which are described as bandits where \say{the objects continue to change state even when they are not being operated}. 
The Bayesian formulation has motivated research in non-stationarity bandit learning that also adopts a \emph{Bayesian viewpoint}, as we do here. Among these studies, 
\cite{slivkins2008adapting} investigate what they call a dynamic multi-armed bandit, in which \say{the reward functions stochastically and gradually change in time.} \cite{mellor2013thompson} describe non-stationary bandits as ones where \say{the environment changes
over time}, and more specifically, bandits where \say{arms change their expected rewards.}  



We found that this definition cannot unambiguously distinguish non-stationary bandits from stationary ones.
As demonstrated in Example~\ref{example:easily} of Section~\ref{section:issues}, this definition can paradoxically categorize the same bandit as both stationary and non-stationary. This ambiguity stems from the fact that this definition depends on the concept of \say{reward distributions,} which can allow for multiple valid sequences of \say{reward distributions} to be associated with the same bandit. Consequently, the existing definition can lead to confusion and inconsistency when attempting to classify bandits as either stationary or non-stationary.

Furthermore, the existing definition has resulted in several consequential issues. Firstly, it has led to the adoption of regret notions that are less effective in assessing agent performance.  
Secondly, the existing definition has led to the design of agents that explore excessively. In what follows, we delve into each of these issues in greater depth.

Dynamic regret \citep{bogunovic2016time, min2023} and weak regret \citep{anantharam1987asymptotically, liu2011logarithmic, liu2012learning, tekin2012online} are two widely-used performance metrics in non-stationary bandit learning. However, we found that these regret notions are ineffective in assessing agent performance and differentiating optimal agents from the rest in some cases. For example, in an intuitively nearly-stationary bandit, an optimal agent may exhibit a large and linear dynamic regret, misleadingly implying poor performance. Conversely, a sufficiently small positive weak regret does not guarantee near-optimal performance, because an optimal agent can incur a substantial, linear, and negative weak regret.
These issues arise because the weak regret and dynamic regret are inconsistent with the conventional notion used for stationary bandits \citep{lai1985asymptotically, neu2022lifting}. Specifically, in a stationary bandit, the dynamic regret can significantly deviate from the conventional notion of regret; in a non-stationary bandit
which closely resembles a stationary one, 
the weak regret and dynamic regret 
can differ considerably from the conventional notion in the stationary bandit that the non-stationary bandit closely resembles. 
These findings highlight the limitations of using dynamic regret and weak regret for assessing agent performance. It also calls for a formal definition of stationarity and non-stationarity that motivates the alignment of regret notions with the conventional notion.


The existing definition of non-stationary bandits has also led to the development of algorithms that can explore excessively. 
Particularly, many variants of Thompson sampling (TS) \citep{Thompson1933} developed for non-stationary bandits \citep{9194367, ghatak2021kolmogorov, gupta2011thompson, mellor2013thompson, raj2017taming, trovo2020sliding, viappiani2013thompson} have been found by \cite{liu2022nonstationary} to exhibit a tendency to over-explore in certain non-stationary bandits. 
This issue is significantly attributed to the dependence of bandit models on a latent sequence, which leads to inconsistencies between these TS variants and the original TS algorithm. 
Specifically, we can construct a nearly-stationary bandit and a stationary bandit that it closely resembles; 
when these TS variants are applied to either bandit, 
they exhibit notable differences in behavior and performance compared to TS applied to the stationary bandit. 
These findings highlight the inconsistencies between the TS variants and the original TS approach; they emphasize the necessity of formally defining stationarity and non-stationarity, which in turn motivates the development of more effective algorithms that align with their stationary counterparts. 

To address these issues, we introduce a formal definition of non-stationary bandits. This definition does not depend on the \say{reward distributions} or other latent sequences and can, therefore, unambiguously classify bandits. 
In addition, for two bandits that provide agents with indistinguishable experiences, our definition classifies them as both stationary or both non-stationary. 
By establishing this formal definition, we can motivate alternative notions of regret and algorithms that avoid the limitations associated with weak regret, dynamic regret, and many variants of TS. We provide concrete examples illustrating these concepts. 

Importantly, our definition of non-stationary bandits seamlessly applies to both Bayesian and frequentist formulations, encompassing restless bandits \citep{whittle1988restless} and bandits with oblivious adversaries \citep{slivkins2019introduction}, respectively, thereby providing a unified framework. This unified framework enhances our understanding of non-stationary bandits and opens up possibilities for the development of more effective notions of regret and algorithms to address the challenges inherent in non-stationary bandits.




\section{Existing Definitions Ambiguously Classify Bandits}
\label{section:nonstationarity}

We begin this section by introducing basic definitions and notations that we will use throughout the paper. 
We start by presenting a formal definition of a \emph{bandit}, drawing inspiration from the formulation of restless bandits \citep{whittle1988restless}. 

\begin{definition}
A bandit with action set $\actions$ is a stochastic process $\{R_t\}_{t \in \mathbb{N}}$ with state space $\R^{|\actions|}$. 
\label{definition:bandit_1}
\end{definition}
In a bandit $\{R_t\}_{t \in \mathbb{N}}$ with a finite action set $\actions$, the \emph{reward} $R_{t+1,a}$ represents what an agent receives upon executing \emph{action} $a \in \actions$ at timestep $t \in \mathbb{N}_0$. We use $\mathbb{N}_{0}$ to denote the set of all nonnegative integers, and $\mathbb{N}$ to denote the set of all positive integers. 
We sometimes use $R_{1:+\infty}$ to denote the sequence of all rewards $\{R_t\}_{t \in \mathbb{N}}$ and $R_{i:j}$ to denote the sequence of rewards between timesteps $i$ and $j$, i.e,  $\{R_t\}_{t \in [j]/[i]}$, where $[i] = \{1, ..., i\}$.

\label{section:issues}

A formal definition of non-stationarity that distinguishes non-stationary bandits from stationary ones is lacking. 
Instead, previous works \citep{abbasi2022new, auer2018adaptively, bogunovic2016time, burtini2015improving, cao2019nearly, chen2019new,  cheung2019learning, di2020linear, garivier2008upper, 9194367, guha2010approximation, gupta2011thompson, hartland2007change, jia2023smooth, jung2019regret, liu2018change, luo2018efficient, mellor2013thompson, min2023, raj2017taming, slivkins2008adapting, trovo2020sliding, viappiani2013thompson, wei2018abruptly, whittle1988restless, wu2018learning, yu2009piecewise,  xu2023bayesian, zhao2020simple, zhou2021regime} 
 refer to a bandit as non-stationary when its \say{reward distribution} changes over time. 
The reward distributions are involved in this definition because, in many cases, rewards are generated according to a particular model, and the reward distributions represents one such model. In other cases, the reward distributions can serve as convenient abstractions of how rewards are generated. 


To apply the existing definition of non-stationarity for classifying bandits, it is necessary to establish a formal definition of a valid reward distribution. Thus, we present the following formal definition. We use $\mathcal{P}(S)$ to denote the set of all distributions over a set $S$. 
\begin{definition} [\bf{Reward Distribution of a Bandit $\{R_t\}_{t \in \mathbb{N}}$}] 
A stochastic process $\{Q_t\}_{t \in \mathbb{N}}$ with state space 
$\mathcal{P}(\R^{|\actions|})$
is a sequence of reward distributions 
if, conditioned on $\{Q_t\}_{t \in \mathbb{N}}$,  $\{R_t\}_{t \in \mathbb{N}}$ are independent, and each $R_t$ is distributed according to $Q_t$. 
\label{definition:reward_distribution}
\end{definition}
It is important to note that any bandit $\{R_t\}_{t \in \mathbb{N}}$ has a trivial sequence of reward distributions: consider $\{Q_t\}_{t \in \mathbb{N}}$ such that $Q_t = \delta_{R_t}$, a Dirac delta function centered at $R_t$, for all $t \in \mathbb{N}$.


We can now note the existing definition of non-stationary bandits in Existing Definition~\ref{definition:informal_nonstationary}. 
\begin{edefinition} [\bf{Non-Stationary Bandit, Existing}] A bandit $\{R_t\}_{t \in \mathbb{N}}$ with reward distributions $\{Q_t\}_{t \in \mathbb{N}}$ is non-stationary if it is not stationary; it is stationary if $Q_t = Q_{t+1}$ a.s. for all $t \in \mathbb{N}$. 
\label{definition:informal_nonstationary}
\end{edefinition}
We can now apply Existing Definition~\ref{definition:informal_nonstationary} to classify bandits.  
We will use an example to illustrate that this definition can lead to ambiguity, classifying the same bandit as both stationary and non-stationary in some cases. 
Below we provide a formal definition of a class of Bernoulli bandits followed by the specific example we consider. 

\begin{definition} [\bf{Bernoulli Bandit with Independent Actions}]
\label{definition:bernoulli}
For all $a \in \actions$, let $\{\theta_{t, a}\}_{t \in \mathbb{N}}$ be an independent stochastic process with state space $[0, 1]$. 
Conditioned on the collection of $\{\theta_{t, a}\}_{t 
 \in \mathbb{N}}$ for all $a \in \actions$, a reward $R_{t, a}$ is distributed according to $\mathrm{Bernoulli}(\theta_{t,a})$, independent of the rewards associated with other timesteps or actions. 
\end{definition}
\begin{example} [\bf{A Hard-To-Classify Bernoulli Bandit}]
Consider a Bernoulli bandit with two independent actions $1$ and $2$, where 
$\theta_{t,1} = 0.8$ for all $t \in \mathbb{N}$ and $\{\theta_{t,2}\}_{t \in \mathbb{N}}$ is i.i.d. according to a uniform distribution over the set $\{0, 1\}$.  
\label{example:easily}
\end{example}

Existing Definition~\ref{definition:bandit_1} would ambigously classify Example~\ref{example:easily} as both stationary and non-stationary. To elaborate, we can take each reward distribution $Q_{t,a}$ to be a Bernoulli distribution determined by its mean $\theta_{t,a}$.  
Consequently, the reward distribution $Q_t$ changes over time: $Q_t \neq  Q_{t+1}$ with probability $0.5$, because $\theta_t \neq \theta_{t+1}$ with probability $0.5$. 
Thus, 
Existing Definition~\ref{definition:bandit_1} would classify Example~\ref{example:easily} as non-stationary. 
On the other hand, 
we can alternatively define each $Q_{t,2}$ as a Bernoulli distribution with mean $0.5$, because $\{R_{t,2}\}_{t \in \mathbb{N}}$ are i.i.d. $\mathrm{Bernoulli}(0.5)$.  
In this case, the reward distribution $Q_{t}$ remains unchanged over time, and Existing Definition~\ref{definition:informal_nonstationary} would classify Example~\ref{example:easily} as stationary.


This ambiguity in classifying bandits arises in bandits with continuous rewards as well. We provide a formal definition of a class of Gaussian bandits followed by a specific example that we consider. 
\begin{definition} [\bf{Gaussian Bandit with Independent Actions}]
\label{example:gaussian}
For all $a \in \actions$, let 
$\{\theta_{t, a}\}_{t \in \mathbb{N}}$ be an independent stochastic process with state space $\mathbb{R}$. 
Conditioned on the collection of $\{\theta_{t, a}\}_{t 
 \in \mathbb{N}}$ for all $a \in \actions$, the reward $R_{t, a}$ is distributed according to $\mathcal{N}(\theta_{t,a}, 1)$, independent of the rewards associated with other timesteps or actions. 
\end{definition}

\begin{example} [\bf{A Hard-To-Classify Gaussian Bandit}]
Consider a Gaussian bandit with two independent actions $1$ and $2$, where 
$\theta_{t,1} = 0.8$ for all $t \in \mathbb{N}$ and $\{\theta_{t,2}\}_{t \in \mathbb{N}}$ is i.i.d. according to a Gaussian distribution $\mathcal{N}(0.5, 1)$.  
\label{example:easily_g}
\end{example}
Example~\ref{example:easily_g} can be classified as 
non-stationary if we take the reward distribution $Q_{t,2}$ to be $\mathcal{N}(\theta_{t,2}, 1)$, and 
stationary if we take $Q_{t,2}$ to be $\mathcal{N}(0.5, 2)$ instead.


This ambiguity in classifying bandits stems from the dependence of Existing Definition~\ref{definition:informal_nonstationary} on the concept of reward distribution. 
Specifically, at least two different $\{Q\}_{t \in \mathbb{N}}$ are valid reward distributions (Definition~\ref{definition:reward_distribution}) for each of Examples~\ref{example:easily} and~\ref{example:easily_g}. 
A more thoughtful definition of non-stationary bandits would resolve this issue. Indeed, we will propose a definition that does not rely on the concept of reward distribution.

\section{Issues in Regret and Agent Design}
\label{section:regret}
This section discusses issues that arise in defining notions of regret.  

We first introduce useful concepts and notions. We refer to each finite sequence of action-reward pairs as a \emph{history}. We use $\histories$ to denote the set of all \emph{histories}. Below we provide a formal definition of a \emph{policy}. 
\begin{definition} [\bf{Policy}]
A policy $\pi: \histories \rightarrow \mathcal{P}(\actions)$ is a function that maps each history in $\mathcal{H}$ to a probability distribution over the action set $\actions$.  
\end{definition}
An \emph{agent} is capable of selecting and executing a policy, while an \emph{algorithm} refers to the set of procedures or steps employed to execute a particular policy. 
For all policies $\pi$, 
we use $A_t^{\pi}$ to denote the action selected at time $t$ by an agent that executes $\pi$, and 
$H^{\pi}_{t}$ to denote the history generated at timestep $t$ as an agent executes $\pi$.
Specifically, at each timestep $t$, we let $H^{\pi}_{t} = (A^{\pi}_0, R_{1,A^{\pi}_0}, \ldots, A^{\pi}_{t-1}, R_{t,A^{\pi}_{t-1}})$; we let 
$A_t^{\pi}$ be such that $\Pr(A^{\pi}_t \in \cdot |H^{\pi}_t) =  \pi(\cdot|H^{\pi}_t)$ and that $A_t^{\pi}$ is independent of $R_{1:+\infty}$ conditioned on $H_t^{\pi}$.


\subsection{Desired Properties of Regret}
Bandit learning agents seek to maximize cumulative expected reward. However, evaluating an agent's performance solely based on the reward collected is challenging because it is difficult to distinguish between good and bad outcomes. To address this, it is helpful to 
compare an agent's performance against a benchmark $ \mathrm{Benchmark}$. This comparison is quantified using the notion of regret. 
Specifically, the regret of a policy $\pi$ over $T \in \mathbb{N}$ timesteps in a bandit $\{R_t\}_{t \in \mathbb{N}}$ is defined as 
\begin{align*}
\mathrm{Regret}(\{R_t\}_{t \in \mathbb{N}}; T; \pi) = \mathrm{Benchmark}(\{R_t\}_{t \in \mathbb{N}}; T) - \sum_{t = 0}^{T-1} \mathbb{E}\left[R_{t+1, A_t^{\pi}}\right].
\end{align*}
When it is clear from the context that we consider a single bandit
$\{R_t\}_{t \in \mathbb{N}}$, we use 
$\mathrm{Benchmark}(T)$ and 
$\mathrm{Regret}(T; \pi)$ as shorthands.  

The construction of $\mathrm{Benchmark}$ plays a crucial role in evaluating agent performance. Ideally, we want to construct $\mathrm{Benchmark}$ 
such that $\mathrm{Benchmark}(\{R_t\}_{t \in \mathbb{N}}; T)$ corresponds to the cumulative expected reward collected by an optimal agent in $\{R_t\}_{t \in \mathbb{N}}$ over $T$ timesteps. 
In the context of stationary bandits, the conventional benchmark, known as $\mathrm{Benchmark}^{\mathrm{C}}$, has gained widespread adoption and is fully satisfactory, 
because $\mathrm{Benchmark}^{\mathrm{C}}$ 
is only slightly better than, and asymptotically equivalent to, the performance of an optimal agent. Thus, $\mathrm{Benchmark}^{\mathrm{C}}$ is already capable to measure the sub-optimality of an agent: a small regret indicates that an agent performs well, while a large regret suggests the opposite. 

However, the existence of non-stationarity poses a significant challenge when using the performance of an optimal agent as our benchmark. This approach, although useful, can complicate the analysis of regret. To address this issue, we aim to construct an alternative benchmark such that the regret is both easy to analyze and capable of effectively evaluating agent performance. 



To effectively evaluate agent performance in non-stationary bandit learning, a notion of regret is expected to be consistent with the conventional regret used in stationary bandit learning. 
This consistency ensures that the two notions are equivalent when applied to a stationary bandit $\{R_t\}_{t \in \mathbb{N}}$. Furthermore, in an intuitively nearly-stationary bandit $\{R^{\prime}_t\}_{t \in \mathbb{N}}$ that closely resembles $\{R_t\}_{t \in \mathbb{N}}$, the regret notion should not deviate significantly from the conventional notion in $\{R_t\}_{t \in \mathbb{N}}$. By maintaining this consistency, we can effectively evaluate agent performance in $\{R^{\prime}_t\}_{t \in \mathbb{N}}$. 

We introduce the following concept of convergence, and formalize this desired property. 

\begin{definition} [\bf{Convergence of Bandits}]
\label{defi:conv}
A sequence $\{\{R_t^{(k)}\}_{t \in \mathbb{N}}\}_{k \in \mathbb{N}}$  of bandits \textit{converges in distribution} to a bandit $\{R_t\}_{t \in \mathbb{N}}$ if and only if for all $m \in \mathbb{N}$ and all sequences $(t_1, \cdots, t_m)\in \mathbb{N}^m$, we have 
\[
\lim_{k \rightarrow +\infty} \mathbb{P}\left(\{R_t^{(k)}\}_{t \in \{t_1, ... , t_m\}}\in \cdot\right) = \mathbb{P}(\{R_{t}\}_{t \in \{t_1, ... , t_m\}}\in \cdot). 
\]
\begin{remark}
Definition \ref{defi:conv} directly follows from the Kolmogorov extension theorem, which ensures that a sufficiently \say{consistent} set of finite-dimensional distributions can define a stochastic process \citep{kimme1957convergence}.
\end{remark}
\label{definition:convergence}
\end{definition}

Next, we formalize the desired property using convergence of bandits. 

\begin{definition}[\bf{Consistency with the Conventional Notion of Regret}]
We say that a notion of regret $\mathrm{Regret}$, defined with respect to $\mathrm{Benchmark}$, is \textit{consistent} with the conventional notion, if 
for all stationary bandits 
$\{R_t\}_{t \in \mathbb{N}}$, and 
for all sequences of bandits $\{\{R_t^{(k)}\}_{t \in \mathbb{N}}\}_{k \in \mathbb{N}}$ that converges to $\{R_t\}_{t \in \mathbb{N}}$, and for all horizon $T \in \mathbb{N}$, 
$\mathrm{Benchmark}$ satisfies 
\begin{align*}
\lim_{k \rightarrow +\infty} \mathrm{Benchmark}\left(\left\{R^{(k)}_t\right\}_{t \in \mathbb{N}}; T\right) = \mathrm{Benchmark}^{\mathrm{C}}(\{R_t\}_{t \in \mathbb{N}}; T). 
\end{align*}
\label{definition:generalize_regret}
\end{definition}

In the upcoming discussion, we will introduce two commonly used notions of regret: dynamic regret and weak regret. 
Through illustrative examples, we will demonstrate that these notions of regret do not satisfy the aforementioned desired property---in other words, are not consistent with the conventional notion of regret defined for stationary bandits. 

\subsection{Issues of Dynamic Regret}
Dynamic regret $\mathrm{Regret}^{\mathrm{D}}$
is defined relative to the performance of an oracle who knows some latent variable, which is usually the reward distribution $Q_t$, at each timestep $t$ and acts optimally. 
We present a formal definition of dynamic regret below, where we take the latent variable at timestep $t \in \mathbb{N}$ to be $Q_t$. 
\begin{edefinition} [\bf{Dynamic Regret}]
\label{definition:dynamic_regret}
For all policies $\pi$, bandits $\{R_t\}_{t \in \mathbb{N}}$, and $T \in \mathbb{N}$, 
the dynamic regret is defined as 
\begin{align*}
\mathrm{Regret}^{\mathrm{D}}(\{R_t\}_{t \in \mathbb{N}}; T;\pi) = 
\mathrm{Benchmark}^{\mathrm{D}}(\{R_t\}_{t \in \mathbb{N}}; T) - 
\sum_{t = 0}^{T-1}\E\left[R_{t+1, A_t^{\pi}}\right], 
\end{align*}
where $\mathrm{Benchmark}^{\mathrm{D}}(\{R_t\}_{t \in \mathbb{N}}; T) = \sum_{t = 0}^{T-1} \mathbb{E}[\max_{a \in \actions}\mathbb{E}[R_{t+1,a}|Q_{t+1}]]$, and $\{Q_t\}_{t \in \mathbb{N}}$ is a sequence of reward distributions. 
\end{edefinition}
A limitation of dynamic regret is that its performance benchmark can be large. 
This is because, given past information, it may be challenging, if not impossible, to accurately learn the reward distribution in the previous timestep, let alone the current reward distribution. As a consequence, even a competent agent may exhibit a substantial dynamic regret, rendering this metric ineffective in distinguishing capable agents in certain scenarios. This limitation is particularly pronounced in nearly-stationary bandits, where the dynamic regret can significantly diverge from the conventional notion of regret defined for stationary bandits. Using such cases, we can show that the dynamic regret is inconsistent with the conventional notion of regret defined for stationary bandits. 

To illustrate this, below we introduce an example. \cite{liu2022nonstationary, mellor2013thompson} introduce bandits of similar structures. 

\begin{example} [\bf{A Modulated Bernoulli Bandit with Independent Actions}]
\label{ex:modulated_bernoulli}
\noindent 
Consider a Bernoulli bandit 
$\{R_t(q)\}_{t \in \mathbb{N}}$
with two independent actions (Definition~\ref{definition:bernoulli}), where 
$\theta_{t,1}(q) = 0.8$ for all $t \in \mathbb{N}$ and 
$\theta_{1,2}(q)$ is distributed uniformly over the set $\{0, 1\}$. 
The sequence $\{\theta_{t,2}(q)\}_{t \in \mathbb{N}}$
transitions according to 
\begin{align*}
    \theta_{t+1, 2}(q) = 
    \begin{cases}
    \sim \mathrm{unif}(\{0, 1\}),  &\ \text{with probability $q$}, \\
    \theta_{t,2}(q), &\ \text{otherwise}, 
    \end{cases}
\end{align*}
where $q \in [0, 1]$ is deterministic. 
At each timestep $t \in \mathbb{N}$, 
$\theta_{t,2}(q)$
can be thought of as \say{redrawn} uniformly from $\{0, 1\}$ 
independently with probability $q$. 
\label{example:bernoulli_example}
\end{example}


This example illustrates that the dynamic regret is not consistent with the conventional notion of regret, i.e., the dynamic regret does not satisfy Definition~\ref{definition:generalize_regret}. First, for all $q \in [0, 1]$, let $Q_t(q)$ to be the distribution of two independent Bernoulli random variables with mean $\theta_{t, 1}(q) = 0.8$ and $\theta_{t,2}(q)$, respectively. Therefore, $\{Q_t(q)\}_{t \in \mathbb{N}}$ is a valid sequence of reward distributions for $\{R_t(q)\}_{t \in \mathbb{N}}$. 
Consequently, for all $q \in [0, 1]$, the performance benchmark of the dynamic regret 
\begin{align*}
\mathrm{Benchmark}^{\mathrm{D}}\left(\{R_t(q)\}_{t \in \mathbb{N}}; T\right) = \sum_{t = 0}^{T-1}\mathbb{E}\left[\max_{a \in \actions} \theta_{t+1,a}(q)\right] = 
\sum_{t = 0}^{T-1}\mathbb{E}\left[\max\{0.8, \theta_{t+1, 2}(q)\}\right]
= 
0.9T.
\end{align*}
However, if $q = 1$, the conventional benchmark $\mathrm{Benchmark}^{\mathrm{C}}(\{R_t(q)\}_{t \in \mathbb{N}}; 
T) = \mathrm{Benchmark}^{\mathrm{C}}(\{R_t(1)\}_{t \in \mathbb{N}}; 
T) = 0.8T$. Thus, Definition~\ref{defi:conv}
 does not hold and the dynamic regret is inconsistent with the conventional notion of regret defined for stationary bandits. As a result, an optimal agent incurs a large and linear dynamic regret in bandits where $q$ is close to $1$ (nearly-stationary), and the dynamic regret of an optimal agent converges to $0.1 T$ as $q \rightarrow 1$.   
 
It is worth highlighting that when $q = 1$, Example~\ref{example:bernoulli_example} is reduced to Example~\ref{example:easily}, for which the existing definition fails to unambiguously classify. So a definition of non-stationarity that can unambiguously classify bandits can motivate notions of regret that are consistent with the conventional notion.

While Example~\ref{example:bernoulli_example} characterizes bandits with \emph{abrupt changes}, let us consider another example,  where the changes are \emph{smooth}. Bandits of similar structures have been studied by 
\cite{burtini2015improving, 
gupta2011thompson, gaussianar1, kuhn2015wireless, liu2022nonstationary, slivkins2008adapting}.

\begin{example} [\bf{An AR(1) Bandit with Independent Actions}]
\label{ex:modulated_bernoulli}
\noindent 
Consider a Gaussian bandit
$\{R_t(\gamma)\}_{t \in \mathbb{N}}$
with two independent actions (Definition~\ref{example:gaussian}), where 
$\theta_{t,1} = 0.8$ for all $t \in \mathbb{N}$ and 
$\theta_{1,2}$ is distributed according to $\mathcal{N}(0.5, 1)$. 
The sequence $\{\theta_{t,2}\}_{t \in \mathbb{N}}$
transitions according to 
\begin{align*}
    \theta_{t+1, 2} = \gamma \theta_{t,2} + (1 - \gamma) W_{t+1}, 
\end{align*}
where $\gamma \in [0, 1]$ is deterministic, and the sequence $\{W_t\}_{t \in \mathbb{N}}$ is i.i.d.  $\mathcal{N}\left(0.5, \frac{1 - \gamma^2}{(1 - \gamma)^2}\right)$. 
\label{example:gaussian_example}
\end{example}

We can show that for all $\gamma \in [0, 1]$, the benchmark $\mathrm{Benchmark}^{\mathrm{D}}(
\{R_t(\gamma)\}_{t \in \mathbb{N}}; 
T) \approx 1.07T$. 
When $\gamma = 0$, Example~\ref{example:gaussian_example} is reduced to Example~\ref{example:easily_g}, and the conventional benchmark $\mathrm{Benchmark}^{\mathrm{C}}(
\{R_t(0)\}_{t \in \mathbb{N}}; 
T) = 0.8T$. 
Therefore, the dynamic regret is inconsistent with the conventional regret, and an optimal agent incurs a large and linear dynamic regret in a nearly-stationary bandit. 

\subsection{Issues of Weak Regret}
Weak regret assesses agent performance relative to the performance of the best single-action policy with respect to the initial reward distribution, that is, the best policy among those that knows the initial reward distribution and selects the same single action at each timestep.  
We present a formal definition of weak regret below. 

\begin{edefinition} [\bf{Weak Regret}]
For all policies $\pi$, bandits $\{R_t\}_{t \in \mathbb{N}}$, and $T \in \mathbb{N}$, 
the weak regret is defined as 
\begin{align*}
\mathrm{Regret}^{\mathrm{W}}(\{R_t\}_{t \in \mathbb{N}}; T; \pi)
=
\mathrm{Benchmark}^{\mathrm{W}}(\{R_t\}_{t \in \mathbb{N}}; T) 
-
\mathbb{E}\left[R_{t+1, A_t^{\pi}}\right], 
\end{align*}
where $\mathrm{Benchmark}^{\mathrm{W}}(\{R_t\}_{t \in \mathbb{N}}; T) = T \mathbb{E}\left[\max_{a \in \actions} \left(\lim_{H \rightarrow +\infty} \frac{1}{H}\sum_{t = 0}^{H-1} \E[R_{t+1,a} | Q_1]\right)\right]$. 
\label{definition:weak_regret}
\end{edefinition}
A limitation of the weak regret is that its performance benchmark $\mathrm{Benchmark}^{\mathrm{W}}$ is conservative, because an optimal policy does not necessarily select the same single action at each timestep. 
This limitation has been acknowledged in prior work, but to our knowledge, no studies have characterized the degree of conservatism of weak regret, or in which types of bandits this issue arises and why. 

We demonstrate that, strikingly, the weak regret is inconsistent with the conventional notion of regret, i.e., does not satisfy Definition~\ref{definition:generalize_regret},  resulting in its conservation and ineffectiveness even in nearly-stationary bandits. 
To illustrate this, we reconsider Example~\ref{example:bernoulli_example}. When $q \in (0, 1]$, the performance benchmark of weak regret 
\begin{align*}
\mathrm{Benchmark}^{\mathrm{W}}\left(\{R_t(q)\}_{t \in \mathbb{N}}; T\right) = T \mathbb{E}[\max\{0.8, 0.5\}] = 0.8 T.
\end{align*}
When $q = 0$, however, the conventional benchmark $\mathrm{Benchmark}^{\mathrm{C}}\left(\{R_t(q)\}_{t \in \mathbb{N}}; T\right) = 
\mathrm{Benchmark}^{\mathrm{C}}\left(\{R_t(0)\}_{t \in \mathbb{N}}; T\right) = 
0.9 T$. 
Therefore, the weak regret is inconsistent with the conventional regret. Additionally, an optimal agent incurs a significant and negative linear weak regret in bandits where $q$ is close to $0$ (nearly-stationary), and the regret converges to $-0.1T$ as $q \rightarrow 0$. 

We also reconsider Example~\ref{example:gaussian_example}, to illustrate that the weak regret is inconsistent with the conventional notion of regret in  bandits with \emph{smooth} changes. When $\gamma \in [0, 1)$, benchmark $\mathrm{Benchmark}^{\mathrm{W}}(
\{R_t(\gamma)\}_{t \in \mathbb{N}}; 
T) = T\mathbb{E}[\max\{0.8, 0.5\}] = 0.8T$. However, when $\gamma = 1$, the conventional benchmark is $\mathrm{Benchmark}^{\mathrm{C}}(
\{R_t(1)\}_{t \in \mathbb{N}}; 
T) \approx 1.07 T$. Therefore, the weak regret is inconsistent with the conventional notion of regret, and an optimal agent can incur a substantial, negative, and linear weak regret in a nearly-stationary bandit.



\subsection{Generalization}
\label{section:generalization}
As discussed earlier, Examples~\ref{example:bernoulli_example} and~\ref{example:gaussian_example} illustrate how dynamic regret and weak regret are inconsistent with the conventional notion of regret defined for stationary bandits. We now demonstrate that these inconsistencies are not limited to those specific examples but exist among a broader set of scenarios. To do so, we introduce a class of Bernoulli bandits and a class of Gaussian bandits, which encompass Examples~\ref{example:bernoulli_example} and~\ref{example:gaussian_example}, respectively. Through our analysis, we reveal that such inconsistencies persist across many sequences of bandits within these classes.

We first introduce the two class of bandits. Modulated Bernoulli bandits represents a class of Bernoulli bandits modulated by a Markov chain. 
\begin{example} [\bf{Modulated Bernoulli Bandit with Independent Actions}]
\label{ex:modulated_bernoulli}
\noindent 
Consider a Bernoulli bandit with independent actions (Definition~\ref{definition:bernoulli}), where for each $a \in \actions$, the sequence $\{\theta_{t,a}\}_{t \in \mathbb{N}}$
transitions according to 
\begin{align*}
    \theta_{t+1, a} = 
    \begin{cases}
    \sim \Pr(\theta_{1,a} \in \cdot),  &\ \text{with probability $q_a$}, \\
    \theta_{t,a}, &\ \text{otherwise}, 
    \end{cases}
\end{align*}
where $q_a \in [0, 1]$ is deterministic. 
At each timestep $t \in \mathbb{N}$, 
$\theta_{t,a}$
can be thought of as ``redrawn" from its initial distribution 
$\Pr(\theta_{1,a} \in \cdot)$ 
independently with probability $q_a$. 
\label{example:modulated}
\end{example}
This formulation encompasses stationary Bernoulli bandits and Example~\ref{example:bernoulli_example}. When $q_a = 0$ for all $a \in \actions$, we recover a stationary Bernoulli bandit; when $\actions = \{1, 2\}$, $\theta_{1, 1} = 0.8$, $q_1 = 0$, and $\theta_{1,2} \sim \mathrm{unif}(\{0, 1\})$, we recover Example~\ref{example:bernoulli_example}.  
Furthermore, this formulation bears a close resemblance to the non-stationary Bernoulli bandits introduced by \cite{liu2022nonstationary, mellor2013thompson}, as well as bandits with \say{abrupt changes} or \say{piece-wise stationary} bandits. 

We next introduce the class of Gaussian bandits. We refer to it as the AR(1) bandits because each bandit is modulated by AR(1) processes. 

\begin{example} [\bf{An AR(1) Bandit with Independent Actions}]
\label{ex:modulated_bernoulli}
\noindent 
Consider a Gaussian bandit with independent actions (Definition~\ref{example:gaussian}), where for each $a \in \actions$, the sequence $\{\theta_{t,a}\}_{t \in \mathbb{N}}$
transitions according to 
\begin{align*}
    \theta_{t+1, a} = \gamma_a \theta_{t,a} + (1 - \gamma_a) W_{t+1,a}, 
\end{align*}
where each $\gamma_a \in [0, 1]$ is deterministic, and the sequence $\{W_{t,a}\}_{t \in \mathbb{N}}$ is i.i.d. Gaussian random variables with mean $\mathbb{E}[\theta_{1,a}]$ and variance $\frac{1 - \gamma^2}{(1 - \gamma)^2} \mathrm{Var}(\theta_{1,a})$.  
\label{example:AR1}
\end{example}

This formulation encompasses stationary Gaussian bandits and Example~\ref{example:gaussian_example}. When $\gamma_a = 1$ for all $a \in \actions$, we obtain a stationary Gaussian bandit; when $\actions = \{1, 2\}$, $\theta_{1, 1} = 0.8$, $\gamma_1 = 1$, and $\theta_{1,2} \sim \mathcal{N}(0.5, 1)$, we recover Example~\ref{example:gaussian_example}.  
Furthermore, this formulation bears a close resemblance to the non-stationary Gaussian bandits introduced by \cite{burtini2015improving, 
gupta2011thompson, gaussianar1, kuhn2015wireless, liu2022nonstationary, slivkins2008adapting}, and represents bandits with \say{smooth/continuous changes.} 

We now present Theorem~\ref{theorem:inconsistency}, which establishes that both the dynamic regret and the weak regret are inconsistent with the conventional notion of regret defined for stationary bandits,  
using modulated Bernoulli bandits and AR(1) bandits. 

\begin{restatable}{theorem}{inconsistency}
{\bf{(Inconsistencies of Dynamic Regret and Weak Regret with the Conventional Notion of Regret).}}
\label{theorem:inconsistency}
Suppose $\mathcal{B}$ is a set of modulated Bernoulli bandits (resp. AR(1) bandits) with the same $\Pr(\theta_{1,a} \in \cdot)$ across bandits for all $a \in \actions$.
Then, 
for all sequences of bandits in $\mathcal{B}$, where the $k$-th bandit is denoted by $\{R_t^{(k)}\}_{t \in \mathbb{N}}$ and determined by $q_a^{(k)}$ (resp. $\gamma_a^{(k)}$), we have
\begin{enumerate}
\item if $\lim_{k \rightarrow +\infty} q_a^{(k)} = 1$ (resp. $\lim_{k \rightarrow +\infty} \gamma_a^{(k)} = 0$) for all $a \in \actions$, then there exists a stationary bandit $\{R_t\}_{t \in \mathbb{N}}$ such that the sequence of bandits converges to $\{R_t\}_{t \in \mathbb{N}}$, and for all $T \in \mathbb{N}$, the performance benchmark of dynamic regret satisfies
\begin{align*}
\lim_{k \rightarrow +\infty}  \mathrm{Benchmark}^{\mathrm{D}}\left(\left\{R^{(k)}_t\right\}_{t \in \mathbb{N}}; T\right) - \mathrm{Benchmark}^{\mathrm{C}}(\{R_t\}_{t \in \mathbb{N}}; T) 
= T \left(\mathbb{E}\left[\max_{a \in \actions} \theta_{1,a}\right] - \max_{a \in \actions} \mathbb{E}[\theta_{1,a}]\right);
\end{align*}
\item if $\lim_{k \rightarrow +\infty} q_a^{(k)} = 0$ (resp. $\lim_{k \rightarrow +\infty} \gamma_a^{(k)} = 1$)  
for all $a \in \actions$, then there exists a stationary bandit $\{R_t\}_{t \in \mathbb{N}}$ such that the sequence of bandits converges to $\{R_t\}_{t \in \mathbb{N}}$, and for all $T \in \mathbb{N}$, the performance benchmark of weak regret satisifies
\begin{align*}
\lim_{k \rightarrow +\infty}  \mathrm{Benchmark}^{\mathrm{W}}\left(\left\{R^{(k)}_t\right\}_{t \in \mathbb{N}}; T\right) - \mathrm{Benchmark}^{\mathrm{C}}(\{R_t\}_{t \in \mathbb{N}}; T) 
= - T \left(\mathbb{E}\left[\max_{a \in \actions} \theta_{1,a}\right] - \max_{a \in \actions} \mathbb{E}[\theta_{1,a}]\right).
\end{align*}
\end{enumerate}
\end{restatable}

It is important to highlight that $\mathbb{E}[\max_{a \in \actions}\theta_{1,a}] - \max_{a \in \actions}\mathbb{E}[\theta_{1,a}] \geq 0$ due to Jensen's inequality, and the inequality is typically strict. For instance, the inequality is strict when $\Pr(\max_{a \in \actions-a^{\star}} \theta_{1,a} > \theta_{1,a^{\star}}) > 0$, where $a^{\star} \in \arg \max_{a \in \actions} \mathbb{E}[\theta_{1,a}]$, and $\mathcal{A} - a^{\star}$ denotes the set of all actions excluding $a^{\star}$.  
This condition is satisfied, for example, when all $\theta_{1,a}$'s have the same non-atomic support. 

Therefore, Theorem~\ref{theorem:inconsistency} reveals that inconsistencies in notions of regret can be prevalent in various Bernoulli and Gaussian bandits; by applying the theorem to Examples~\ref{example:bernoulli_example} and~\ref{example:gaussian_example}, we can specifically identify the inconsistencies in those examples. It is worth noting that Bernoulli and Gaussian bandits are among the most widely studied bandits in the literature, 
and that modulated Bernoulli bandits and AR(1) bandits resemble the bandits introduced by \cite{burtini2015improving, gupta2011thompson, gaussianar1, kuhn2015wireless, liu2022nonstationary, mellor2013thompson, slivkins2008adapting}, as well as others such as bandits with \say{abrupt changes} or \say{piece-wise stationary} bandits, or bandits with \say{smooth changes.} These findings suggest that the observed inconsistencies are of a general nature.

\section{Issues in Agent Design}
This section discusses issues that arise in designing agents for non-stationary bandits. 
\subsection{Desired Properties of Agent Design}
A number of non-stationary bandit learning algorithms \citep{auer2018adaptively, 
besbes2019optimal, besson2019generalized, chen2019new, cheung2019learning, garivier2008upper, 9194367, ghatak2021kolmogorov, gupta2011thompson, hartland2007change, kocsis2006discounted, luo2018efficient, raj2017taming, russac2019weighted, trovo2020sliding, viappiani2013thompson, wei2021non, wei2018abruptly, yu2009piecewise, zhao2020simple} have been developed, drawing from traditional stationary bandit learning algorithms such as TS, Upper-Confidence Bounds (UCB) \citep{auer2002finite, lai1985asymptotically}, or Exponential-weight algorithms (EXP3) \citep{auer2002nonstochastic, cesa2006prediction}. These non-stationary variants improve upon the traditional algorithms by incorporating heuristics such as maintaining a sliding-window, discounting past rewards, periodic restart, or detecting change-points.

In designing an algorithm for non-stationary bandits, it is desirable for the algorithm to perform well in stationary bandits as well. 
Thus, the algorithm does not sacrifice its performance in stationary bandits to gain advantages in non-stationary ones.
Specifically, for a non-stationary variant $\pi^{\prime}$, it is desirable for it to be consistent with its stationary counterpart $\pi$. 
This consistency ensures that the two policies are equivalent when applied to a stationary bandit $\{R_t\}_{t \in \mathbb{N}}$. 
Furthermore, in an intuitively nearly-stationary bandit $\{R^{\prime}_t\}_{t \in \mathbb{N}}$ that closely resembles $\{R_t\}_{t \in \mathbb{N}}$, $\pi^{\prime}$ should not deviate substantially from $\pi$. By maintaining this consistency, $\pi^{\prime}$ is effective in $\{R_t\}_{t \in \mathbb{N}}$ and 
$\{R^{\prime}_t\}_{t \in \mathbb{N}}$.


Below we formalize the desired property using the convergence of bandits (Definition~\ref{definition:convergence}). According to Definition~\ref{definition:consistency_algorithm}, a desired property of a non-stationary bandit learning algorithm $\pi^{\prime}$ that is designed based on a stationary bandit learning algorithm $\pi$ is for $\pi^{\prime}$ to be consistent with $\pi$. 

\begin{definition} [\bf{Consistency of Policies}]
We say that policy $\pi^{\prime}$ is \emph{consistent} with policy $\pi$ if 
for all stationary bandits 
$\{R_t\}_{t \in \mathbb{N}}$, all sequences of bandits $\{\{R_t^{(k)}\}_{t \in \mathbb{N}}\}_{k \in \mathbb{N}}$ that converges to $\{R_t\}_{t \in \mathbb{N}}$, and all horizon $T \in \mathbb{N}$, we have 
\begin{align*}
\lim_{k \rightarrow +\infty} \sum_{t = 0}^{T-1}\mathbb{E}\left[R^{(k)}_{t+1, A_t^{\pi^{\prime}}}\right] = \sum_{t = 0}^{T-1}\mathbb{E}[R_{t+1, A_t^{\pi}}].
\end{align*}
\label{definition:consistency_algorithm}
\end{definition}


\subsection{Issues in Agent Design}
Taking some non-stationary variants of TS
\citep{9194367, ghatak2021kolmogorov, mellor2013thompson, raj2017taming, trovo2020sliding, viappiani2013thompson}
as examples, we discuss how various non-stationary bandit learning algorithms are inconsistent with their stationary counterparts. 
To illustrate this, again, we consider Example~\ref{example:bernoulli_example}.
At each timestep $t \in \mathbb{N}$, a non-stationary variant of TS estimates the posterior $\mathbb{P}( \theta_{t,a}(q) \in \cdot| H_t)$ using heuristics such as maintaining a sliding-window, discounting past rewards, periodic restart, or detecting change-points. It then samples $\hat{\theta}_{t,a}(q)$ from the estimated posterior $\hat{\mathbb{P}}(\theta_{t,a}(q) \in \cdot | H_t)$ and selects an action that maximizes the sample $A_t \in \arg \max_{a \in \{1, 2\}}\hat{\theta}_{t,a}(q)$. 

In Example~\ref{example:bernoulli_example}, since $\theta_{t,1}(q) = 0.8$ is deterministic, i.e., known, $\hat{\theta}_{t,1}(q) = 0.8$. For all $q \in [0, 1]$, the estimated posterior $\hat{\mathbb{P}}(\theta_{t,2}(q) \in \cdot | H_t)$ places at least $0.5 q$ mass on $1$. 
Therefore, an aforementioned non-stationary variant of TS selects action $2$ with at least $0.5 q$ probability at each timestep. In contrast, if $q = 1$, TS selects action $2$ with probability $0$ at each timestep. Thus, an aforementioned non-stationary variant of TS is inconsistent with TS, i.e., 
does not satisfy Definition~\ref{definition:consistency_algorithm}. This inconsistency also implies that the non-stationary variants excessively explore by selecting action $2$ with a large probability in nearly-stationary bandits ($q$ close to $1$). 

We also consider Example~\ref{example:gaussian_example} as a second example where variants of TS are inconsistent with TS and over-explore in certain non-stationary bandits. If $\gamma = 0$, TS selects action $2$ with probability $0$ at each timestep. 
However, for all $\gamma \in [0, 1]$, the estimated posterior at initial timestep is $\hat{\mathbb{P}}(\theta_{1,2}(\gamma) \in \cdot) = \mathbb{P}(\theta_{1,2}(\gamma) \in \cdot) \sim \mathcal{N}(0.5, 1)$, and a variant of TS selects action $2$ with probability $\approx 0.38$. This inconsistency suggests that these non-stationary variants of TS over-explore in certain cases. 

\section{A Formal Definition of Non-Stationary Bandits}
\label{section:explained}
In this section, we provide a formal definition of non-stationarity that can unambiguously classify bandits. 
We discuss the rationale behind our definition: if two bandits offer an agent indistinguishable experiences, then they should be consistently classified as either both stationary or both non-stationary. 
Our definition applies seamlessly to bandits in the frequentist formulation as well, enabling unified understanding of non-stationarity.

\subsection{Our Definition}
We first introduce stationary bandits, followed by non-stationary bandits. 
\begin{definition} [\bf{Stationary Bandit}] A bandit $\{R_t\}_{t \in \mathbb{N}}$ is stationary if, for all horizon $T \in \mathbb{N}$, all sequences of actions $\{a_k\}_{k \in [T]}$, and all sequences of distinct timesteps $\{t_k\}_{k \in [T]}$ and $\{t^{\prime}_k\}_{k \in [T]}$, the sequences of rewards $\{R_{t_k, a_k}\}_{k \in [T]}$ and $\{R_{t^{\prime}_k, a_k}\}_{k \in [T]}$ are equal in distribution. 
\label{definition:stationary}
\end{definition}
\begin{definition} [\bf{Non-Stationary Bandit}] A bandit is non-stationary if it is not stationary. 
\label{definition:nonstationary}
\end{definition}

In other words, in a stationary bandit, the distribution of a sequence of rewards associated with distinct timesteps and a particular action sequence depends on only the action sequence but not the timesteps; in a non-stationary bandit, the distribution can depend on the timesteps. 


Notably, our definitions do not depend on a notion of reward distribution, and thus they resolve the ambiguity that arises in classifying bandits. For example, observe that in Example~\ref{example:easily}, 
the distribution of the sequence of rewards $\{R_{t_k, a_k}\}_{k \in [T]}$ is the joint distribution of independent Bernoulli random variables, each with mean $0.8$ if $a_k = 1$ and $0.5$ if $a_k = 2$. 
In Example~\ref{example:easily_g}
the sequence of rewards $\{R_{t_k, a_k}\}_{k \in [T]}$ is the joint distribution of independent Gaussian random variables, each with mean $0.8$ if $a_k = 1$ and $0.5$ if $a_k = 2$, and variance $1$. 
These distributions do not depend on the timesteps $\{t_k\}_{k \in [T]}$. Thus, by Definition~\ref{definition:nonstationary}, Examples~\ref{example:easily} and~\ref{example:easily_g} are unambiguously identified as stationary.

\subsection{Consistent Classification under Indistinguishable Experiences}

An alternative approach to defining non-stationary bandits is to 
consider those for which there does not exist a sequence of \say{reward distributions} that does not change over time. 
This is equivalent to defining non-stationary bandits as those for which the reward sequence is not exchangeable, by deFinetti's theorem. 
This alternative definition also addresses the ambiguity in classifying bandits; for example, Examples~\ref{example:easily} and~\ref{example:easily_g} are both stationary under this definition.  

However, we demonstrate that this alternative definition cannot consistently classify two bandits that offer an agent indistinguishable experiences as both stationary or both non-stationary, whereas our definition achieves this. Before formally defining \emph{indistinguishable experiences,} we illustrate this using examples. Consider two Bernoulli bandits, both determined by the same sequences $\{U_{t,a}\}_{t \in \mathbb{N}}$ of i.i.d. uniform random variables over $[0, 1]$, for $a \in \actions$. . 

\begin{example} [\bf{A Bernoulli Bandit with Independent Noises}]
 Consider a Bernoulli bandit $\{R^{\mathrm{ind}}_t\}_{t \in \mathbb{N}}$ for which $R^{\mathrm{ind}}_{t,a} = \mathbf{1}{\{U_{t,a} < 0.5\}}$ for all $t \in \mathbb{N}$ and $a \in \actions$. 
\label{example:independent}
\end{example}

\begin{example} [\bf{A Bernoulli Bandit with Dependent Noises}] Fix some $\overline{a} \in \actions$. 
Consider a Bernoulli bandit $\{R^{\mathrm{dep}}_t\}_{t \in \mathbb{N}}$, where $R^{\mathrm{dep}}_{t,a} = \mathbf{1}{\{U_{t, {a}} < 0.5\}}$ for all odd $t \in \mathbb{\mathbb{N}}$ and $a \in \actions$, and $R^{\mathrm{dep}}_{t,a} = \mathbf{1}{\{U_{t, \overline{a}} < 0.5\}}$ for all even $t \in \mathbb{N}$ and $a \in \actions$. 
\label{example:dependent}
\end{example}

If we interpret each reward $R_{t,a}$ as a noisy realization of $0.5$, then Example~\ref{example:independent} describes a bandit where the noises are always independent across actions, and Example~\ref{example:dependent} describes a bandit where the noises are independent across actions at odd timesteps and dependent at even timesteps. Note that an agent only observes the reward associated with the action it takes at each timestep, not those of the other actions. Therefore, the dependence structure of the noises across actions is irrelevant to an agent's experience, and Examples~\ref{example:independent} and~\ref{example:dependent} offer indistinguishable experiences to an agent. 

We believe that two bandits providing indistinguishable experiences should be classified as both stationary or both non-stationary. 
We would unequivocally all agree that  Example~\ref{example:independent} should be classified as stationary, and both our definition and the alternative definition would correctly classify it.
As for Example~\ref{example:dependent}, our definition desirably classifies it as stationary, because 
the reward sequence $\{R_{t_k, a_k}\}_{k \in [T]}$ is a sequence of i.i.d. $\mathrm{Bernoulli}(0.5)$ random variables.  
In contrast, the alternative definition classifies the bandit as non-stationary because the sequence of rewards is not exchangeable. Our definition consistently classifies Examples~\ref{example:independent} and~\ref{example:dependent} while the alternative one does not. 



\subsection{A Formal Definition of Indistinguishable Experiences}
As previously discussed, we believe that two bandits providing an agent with indistinguishable experiences should be classified as either both stationary or both non-stationary. To demonstrate how our definition meets this criterion beyond Examples~\ref{example:independent} and~\ref{example:dependent}, we introduce a binary relation on the set of all banadits that formalizes the concept of \emph{indistinguishable experience.} 

We introduce some useful notations followed by the binary relation. 
We use $A_t^{\pi \rightarrow R}$
to denote the action selected by an agent at timestep $t \in \mathbb{N}$ when executing $\pi$ in a bandit environment $\{R_t\}_{t \in \mathbb{N}}$, and use 
$H_t^{\pi \rightarrow R}$ to denote the corresponding history generated at timestep $t \in \mathbb{N}$.

\begin{definition} [\bf{Equivalence Relation $\sim$}]
Let $\{R_t\}_{t \in \mathbb{N}}$ and $\{R^{\prime}_t\}_{t \in \mathbb{N}}$ be two bandits. We say $\{R_t\}_{t \in \mathbb{N}} \sim \{R^{\prime}_t\}_{t \in \mathbb{N}}$ if,
for all policies $\pi$ and $T \in \mathbb{N}$, 
$H_T^{\pi \rightarrow R}$ and $H_T^{\pi \rightarrow R^{\prime}}$
 equal in distribution. 
\label{definition:equivalence}
\end{definition}


This relation $\sim$ is an equivalence relation, since it is reflexive, symmetric, and transitive. When two bandit environments $\{R_t\}_{t \in \mathbb{N}} \sim \{R^{\prime}_t\}_{t \in \mathbb{N}}$, we say that they are \emph{equivalent}. By Definition~\ref{definition:equivalence}, two bandits are equivalent if interacting with each of them provides an agent the same distribution over histories. This is what we refer to as \emph{indistinguishable experience.} 

We next present Theorem~\ref{theorem:equivalence}, which establishes that our definition of non-stationarity, i.e., Definition~\ref{definition:nonstationary},  would classify two equivalent bandits as both stationary or both non-stationary.  
\begin{restatable}{theorem}{theoremequiv} 
{\bf{(Equivalent Bandits Have the Same Non-Stationarity).}}
If two bandits $\{R_t\}_{t \in \mathbb{N}} \sim \{R^{\prime}_t\}_{t \in \mathbb{N}}$,  
then bandit $\{R_t\}_{t \in \mathbb{N}}$ is non-stationary if and only if bandit $\{R^{\prime}_t\}_{t \in \mathbb{N}}$ is non-stationary. 
\label{theorem:equivalence}
\end{restatable}

The proof of Theorem~\ref{theorem:equivalence}, and the proofs of other technical results can be found in the Appendix. 



\subsection{Relation to Exchangeability of the Reward Sequence}

Exploring the relationship between our definition of non-stationary bandits and the alternative definition involving the exchangeability of the reward sequence can provide interesting insights. In this regard, we present Theorem~\ref{theorem:iff}, which establishes that a bandit is stationary if and only if the experience it provides to an agent is indistinguishable from that provided by a bandit where the sequence of rewards is exchangeable.

\begin{restatable}{theorem}{iffstationarity}{\bf{(Necessary and Sufficient Condition for Stationarity).}}
A bandit $\{R_t\}_{t \in \mathbb{N}}$ is stationary if and only if there exists a bandit 
$\{R_t^{\prime}\}_{t \in \mathbb{N}}$ such that 
$\{R_t\}_{t \in \mathbb{N}} \sim \{R_t^{\prime}\}_{t \in \mathbb{N}}$ 
and that the sequence of rewards $\{R_t^{\prime}\}_{t \in \mathbb{N}}$ is exchangeable. 
\label{theorem:iff}
\end{restatable}

\subsection{Equivalence Classes and Strongly Stationary Bandits}
Equivalence relation $\sim$ partitions the set of all bandits into distinct equivalence classes: two bandits belong to the same class if and only if they are equivalent. By Theorem~\ref{theorem:equivalence}, the bandits in each class are either all stationary or all non-stationary. 
Although any bandit can represent its own class, it is an intriguing theoretical pursuit to identify more natural representatives for the classes comprising stationary bandits.  
For this purpose, we introduce strongly stationary bandits. 

\begin{definition} [\bf{Strongly Stationary Bandit}]
A stationary bandit is strongly stationary if for all $T \in \mathbb{N}$, $\{R_t\}_{t \in [T]}$ and $\{(R_{n|\actions| + 1, 1}, R_{n|\actions| + 2, 1}, ..., R_{n|\actions| + |\actions|, |\actions|})\}_{n \in [T]}$ equal in distribution. 
\label{definition:strongly_stationary}
\end{definition}

Note that, in a strongly stationary bandit $\{R_t\}_{t \in \mathbb{N}}$, the sequence of rewards is exchangeable. Because $\{R_t\}_{t \in \mathbb{N}}$ is stationary, the sequence of $|\actions|$-tuples $\{(R_{n|\actions| + 1, 1}, R_{n|\actions| + 2, 2}, ..., R_{n|\actions| + |\actions|, |\actions|})\}_{n \in \mathbb{N}_0}$ is exchangeable. 
In addition, $\{R_t\}_{t \in \mathbb{N}}$ is strongly stationary, so
by Definition~\ref{definition:strongly_stationary}, 
 the sequence $\{R_t\}_{t \in \mathbb{N}}$ is exchangeable. 

Theorem~\ref{theorem:strongly_stationary}  characterizes the equivalence classes of stationary bandits and their relation to strongly stationary bandits. 
It establishes that each equivalence class of stationary bandits contains at least one strongly stationary bandit. Furthermore, if two strongly stationary bandits, $\{R_t\}_{t \in \mathbb{N}}$ and $\{R^{\prime}_t\}_{t \in \mathbb{N}}$, belong to the same class, their sequences of rewards have the same distribution. 
This result makes strongly stationary bandits natural representatives of the equivalence classes of stationary bandits.

\begin{restatable}{theorem}{landscape} {\bf{(Landscape of Bandit Environments).}}
 A bandit $\{R_t\}_{t \in \mathbb{N}}$ is stationary if and only if it is equivalent to a strongly stationary bandit; two strongly stationary bandits are equivalent if and only if for all  $T \in \mathbb{N}$, $\{R_t\}_{t \in [T]}$ and $\{R^{\prime}_t\}_{t \in [T]}$ equal in distribution. 
\label{theorem:strongly_stationary}
\end{restatable}

\subsection{Application to Frequentist Setting}
This subsection explores the application of our definition of non-stationarity to a bandit under a frequentist formulation.
To begin, we introduce a formal definition of a bandit under a frequentist formulation.
The definition is adapted from the definition of a bandit with oblivious adversary \citep{slivkins2019introduction} and the definition of a stationary bandit \citep{lattimore2020bandit}. 
\begin{definition} [\bf{Bandit,  Frequentist}]
A bandit with a finite action set $\mathcal{A}$ is a set 
$\mathcal{Q}$ of infinite sequences of distributions, each over $\mathbb{R}^{|\actions|}$. For all $q = (q_1, q_2, ...) \in \mathcal{Q}$, the corresponding sequence of rewards $R_{1:+\infty}$ is independent, and each $R_t$ is distributed according to distribution $q_t$. 
\label{definition:bandit_freq}
\end{definition}

By putting a prior distribution over the set $\mathcal{Q}$, we can construct a bandit $\{R_t\}_{t \in \mathbb{N}}$ under the Bayesian formulation. 
Applying our definition of non-stationarity, 
a bandit 
$\mathcal{Q}$ 
under the frequentist formulation is stationary if and only if for any prior distribution $\nu$ over the set $\mathcal{Q}$, 
the bandit $\{R_t\}_{t \in \mathbb{N}}$ under a Bayesian formulation that is constructed 
by putting $\nu$ over $\mathcal{Q}$
is stationary. 

To gain insights into this, we introduce two  bandits under the frequentist formulation.

\begin{example} [\bf{A Bandit, Frequentist}]
Let $\mathcal{Q}$ be a bandit, where for all $q \in \mathcal{Q}$, $i, j \in \mathbb{N}$, $q_i = q_j$. 
\label{example:stationary_1}
\end{example} 

Example~\ref{example:stationary_1} is stationary, because for any bandit $\{R_t\}_{t \in \mathbb{N}}$ under a Bayesian formulation that is constructed by putting a prior over $\mathcal{Q}$, the sequence of rewards is exchageable. 

\begin{example} [\bf{A Bernoulli Bandit with Dependent Noises, Frequentist}]
Consider a bandit $\mathcal{Q} = \{q\}$, where $q = (q_1, q_2, \cdots)$ and $q_t$ is the distribution of $R_t$ in Example~\ref{example:dependent}. 
\label{example:stationary_2}
\end{example}
Example~\ref{example:stationary_2} is stationary. Although $q_1 \neq q_2$, 
this bandit 
and a different bandit  $\mathcal{Q}^{\prime} = \{q^{\prime}\}$ where $q^{\prime}_1 = q_2^{\prime} = \cdots = q_1$ 
offer indistinguishable experience to 
an agent. 

In summary, we can apply our definition of non-stationarity to classify bandits under a frequentist formulation. Example~\ref{example:stationary_1} is classified as stationary. Some other bandits, e.g. Example~\ref{example:stationary_2}, are also considered stationary. 

\section{Examples of Notion of Regret, Agent Design, and More}
Our formal definition of non-stationarity significantly improves our understanding of non-stationary bandit learning. It not only resolves the ambiguity that arise in classifying bandits but can also facilitate the development of regret notions and agent designs that align with conventional notions and stationary bandit learning algorithms, respectively. This section provides examples that highlight these developments. Since conventional notions and algorithms are designed only for bandits where reward sequence is exchangeable, we restrict our attention to such bandits when discussing consistency. 
\paragraph{Notion of Regret} We first introduce a notion of regret, with respect to any stochastic process $\chi$. 
\begin{definition}
[\bf{Regret}]
For all bandits $\{R_t\}_{t \in \mathbb{N}}$, 
stochastic processes $\chi = \{\chi_t\}_{t \in \mathbb{N}}$, policies $\pi$ and $T \in \mathbb{N}$, the regret of $\pi$ with respect to $\chi$ is defined as
\begin{align*}
\mathrm{Regret}^{\chi}(\{R_t\}_{t \in \mathbb{N}}; T; \pi) = 
\mathrm{Benchmark}^{\chi}(\{R_t\}_{t \in \mathbb{N}}; T) - 
\sum_{t=0}^{T-1} \mathbb{E}\left[R_{t+1, A_t^{\pi}}\right],  
\end{align*}
where 
$\mathrm{Benchmark}^{\chi}(\{R_t\}_{t \in \mathbb{N}}; T) = \sum_{t = 0}^{T-1} \mathbb{E}[\max_{a \in \actions} \E[R_{t+1, a}|\chi_t]]$. 
\label{definition:regret}
\end{definition}
Our notion of regret is flexible and encompasses various existing definitions. 
For instance, 
it is equivalent to dynamic regret \citep{besbes2019optimal, besson2019generalized, cheung2019learning, garivier2008upper, gupta2011thompson, mellor2013thompson, min2023, raj2017taming, slivkins2008adapting} 
 if we set $\chi_t$ to be the reward distribution at timestep $t+1$. 
Furthermore, our notion is equivalent to a regret introduced by \cite{liu2022nonstationary} if we set $\chi_t = R_{1:t}$.  

This notion also already encompasses definitions that are consistent with the conventional notion of regret, and provides room for further exploration and refinement. 
We highlight one such definition: for any bandit $\{R_t\}_{t \in \mathbb{N}}$, 
let $\chi_t = R_{-\infty:t}$. The negative-indexed rewards are constructed by extending the bandit if the reward sequence is reversible, and are constructed arbitrarily otherwise. This is intuitively consistent with the conventional notion: in a stationary bandit with exchangeable rewards, $\mathbb{E}[R_{t+1}|R_{-\infty:t}] = \mathbb{E}[R_{t+1}|P]$, where $P$ is the invariant reward distribution, i.e., the rewards are i.i.d. according to $P$, conditioned on $P$.


\paragraph{Agent Design} Predictive sampling (PS), an algorithm introduced by \cite{liu2022nonstationary}, samples at each timestep $t$ the sequence of future rewards $R_{t+2:+\infty}$, assuming that the sample is the true future rewards, and acts optimally. PS is proven to be equivalent to TS in stationary bandits with exchangeable rewards. 

\paragraph{Regret Analysis} 
To demonstrate the effectiveness of the introduced notion of regret in evaluating agent performance, we conduct an information-theoretic regret analysis \citep{bubeck2015bandit, dong2018information, hao2022contextual, lattimore2019information, lu2023reinforcement, neu2022lifting, RussoMOR2014, russo2016information, russo2018learning}. This analysis builds upon and directly extends the work of \citep{liu2022nonstationary}; it 
applies to any agent and any bandit---stationary or non-stationary. We provide details in the Appendix. 

\section{Conclusion}
This paper introduces a formal definition of non-stationary bandits that umambiguously classify bandits. 
This definition applies 
seamlessly to bandits under both Bayesian and frequentist formulations. 
In addition, we identify inconsistencies and ineffectiveness in some widely-used notions of regret, such as the dynamic regret and weak regret, as well as popular non-stationary bandit learning algorithms, including several variants of TS. 
These notions of regret and algorithms, similarly to the existing definitions of non-stationary bandits, also exhibit dependence on a latent sequence, and are thus ineffective. 
Our definition of non-stationary bandits enhances our understanding of non-stationary bandits and serves as a foundation for future development of more effective notions of regret and algorithms that align with their stationary counterparts. 

\subsection*{Acknowledgement}
This work is partially supported by Army Research Office (ARO) grant W911NF2010055, the MS\&E Fellowship Fund of Stanford University, and Statistics for Improving Insights, Models, and Decisions Grant of Meta.

\bibliographystyle{unsrt}
\small
\bibliography{definition}

\appendix
\section{Organization of the Appendix}
We organize the appendix as follows:
\begin{itemize}
\item The proof of Theorem~\ref{theorem:inconsistency} is given in Section~\ref{section:inconsistency}. 
\item The proof of Theorem~\ref{theorem:equivalence} is given in Section~\ref{section:equivalence}. 
\item The proof of Theorem~\ref{theorem:iff} is given in Section~\ref{section:iff}. 
\item The proof of Theorem~\ref{theorem:strongly_stationary} is given in Section~\ref{section:strongly}. 
\item A regret analysis is provided in Section~\ref{section:analysis}. 
\end{itemize}

\section{Proof of Theorem~\ref{theorem:inconsistency}}
\label{section:inconsistency}
\inconsistency*
\begin{proof}
In the set $\mathcal{B}$, the bandits only differ by $q_a$'s (resp. $\gamma_a$'s). 
We use $\{R^0_t\}_{t \in \mathbb{N}}$ to denote the bandit where $q_a = 0$ (resp. $\gamma_a = 1$) for all $a \in \actions$, and 
$\{R^1_t\}_{t \in \mathbb{N}}$ to denote the bandit where $q_a = 1$ (resp. $\gamma_a = 0$) for all $a \in \actions$. 

Then for all bandits $\{R_t^{\prime}\}_{t \in \mathbb{N}} \in \mathcal{B}$, the performance benchmark of the dynamic regret is equivalent to that of the conventional regret in $\{R_t^{0}\}_{t \in \mathbb{N}}$: 
\begin{align}
\mathrm{Benchmark}^{\mathrm{D}}\left(\left\{R_t^{\prime}\right\}_{t \in \mathbb{N}}; T\right) = \mathrm{Benchmark}^{\mathrm{C}}\left(\left\{R_t^{0}\right\}_{t \in \mathbb{N}}; T\right)
= T\mathbb{E}\left[\max_{a \in \actions} \theta_{1,a}\right]. 
\label{eq:dynamic_regret}
\end{align}
For all bandits $\{R_t^{\prime}\}_{t \in \mathbb{N}} \in \mathcal{B}$, determined by $q_a^{\prime}$ (resp. $\gamma_a^{\prime}$), if $q_a^{\prime} > 0$ (resp. $\gamma_a^{\prime} < 1$), 
the performance benchmark of the weak regret
is equivalent to that of the conventional regret in $\{R_t^1\}_{t \in \mathbb{N}}$:
\begin{align}
\mathrm{Benchmark}^{\mathrm{W}}\left(\left\{R_t^{\prime}\right\}_{t \in \mathbb{N}}; T\right) = \mathrm{Benchmark}^{\mathrm{C}}\left(\left\{R_t^{1}\right\}_{t \in \mathbb{N}}; T\right)
= T \max_{a \in \actions} \mathbb{E}[\theta_{1,a}]. 
\label{eq:weak_regret}
\end{align}
Therefore,  if $q_a^{(k)} \rightarrow 1$ (resp. $\gamma_a^{(k)} \rightarrow 0$) for all $a \in \actions$, then there exists a stationary bandit $\{R_t^1\}_{t \in \mathbb{N}}$ such that the sequence of bandits converges to $\{R_t^1\}_{t \in \mathbb{N}}$, and for all $T \in \mathbb{N}$, the performance benchmark of dynamic regret 
\begin{align*}
&\ \lim_{k \rightarrow +\infty}  \mathrm{Benchmark}^{\mathrm{D}}\left(\left\{R^{(k)}_t\right\}_{t \in \mathbb{N}}; T\right) - \mathrm{Benchmark}^{\mathrm{C}}\left(\left\{R_t^1\right\}_{t \in \mathbb{N}}; T\right) \\
= &\ \mathrm{Benchmark}^{\mathrm{C}}\left(\left\{R^0_t\right\}_{t \in \mathbb{N}}; T\right) - \mathrm{Benchmark}^{\mathrm{C}}\left(\left\{R_t^1\right\}_{t \in \mathbb{N}}; T\right)\\
= &\ T \left(\mathbb{E}\left[\max_{a \in \actions} \theta_{1,a}\right] - \max_{a \in \actions} \mathbb{E}[\theta_{1,a}]\right),
\end{align*}
where the first equality follows from \eqref{eq:dynamic_regret}. 

In addition, if $q_a^{(k)} \rightarrow 0$ (resp. $\gamma_a^{(k)} \rightarrow 1$) for all $a \in \actions$, then there exists a stationary bandit $\{R_t^0\}_{t \in \mathbb{N}}$ such that the sequence of bandits converges to $\{R_t^0\}_{t \in \mathbb{N}}$, and for all $T \in \mathbb{N}$, the performance benchmark of weak regret 
\begin{align*}
&\ \lim_{k \rightarrow +\infty}  \mathrm{Benchmark}^{\mathrm{W}}\left(\left\{R^{(k)}_t\right\}_{t \in \mathbb{N}}; T\right) - \mathrm{Benchmark}^{\mathrm{C}}\left(\left\{R_t^0\right\}_{t \in \mathbb{N}}; T\right) \\
= &\ \mathrm{Benchmark}^{\mathrm{C}}\left(\left\{R^1_t\right\}_{t \in \mathbb{N}}; T\right) - \mathrm{Benchmark}^{\mathrm{C}}\left(\left\{R_t^0\right\}_{t \in \mathbb{N}}; T\right) \\
= &\ - T \left(\mathbb{E}\left[\max_{a \in \actions} \theta_{1,a}\right] - \max_{a \in \actions} \mathbb{E}[\theta_{1,a}]\right),
\end{align*}
where the first equality follows from \eqref{eq:weak_regret}. 
\end{proof}

\section{Proof of Theorem~\ref{theorem:equivalence}}
\label{section:equivalence}
\theoremequiv*
\begin{proof}
Let $T \in \mathbb{N}$, and $\{a_k\}_{k \in [T]}$ be a sequence of actions, and $\{t_k\}_{k \in [T]}$ and $\{t_k^{\prime}\}_{k \in [T]}$ be sequences of distinct timesteps.  

Now consider a policy $\pi$ that executes a deterministic action $a_k$ at timestep $t_k$ for $k \in [T]$, and another policy $\pi^{\prime}$ that executes a deterministic action $a_k$ at timestep $t_k^{\prime}$ for $k \in [T]$. 
Since $\{R_t\}_{t \in \mathbb{N}} \sim \{R^{\prime}_t\}_{t \in \mathbb{N}}$, the histories $H_T^{\pi \rightarrow R}$ and $H_T^{\pi \rightarrow R^{\prime}}$
 have the same distribution; so do $H_T^{\pi^{\prime} \rightarrow R}$ and $H_T^{\pi^{\prime} \rightarrow R^{\prime}}$. Therefore, 
 \begin{align*}
&\ \Pr\left(\left\{R_{t_k, a_k}\right\}_{k \in [T]} \in \cdot\right) =  \Pr\left(\left\{R^{\prime}_{t_k, a_k}\right\}_{k \in [T]} \in \cdot\right),  \\
\text{ and } &\ \Pr\left(\left\{R_{t^{\prime}_k, a_k}\right\}_{k \in [T]} \in \cdot\right) =  \Pr\left(\left\{R^{\prime}_{t^{\prime}_k, a_k}\right\}_{k \in [T]} \in \cdot\right). \numberthis
\label{eq:this}
 \end{align*}
In addition, by the stationarity of $\{R_t\}_{t \in \mathbb{N}}$, 
the reward sequences $\{R_{t_k, a_k}\}_{k \in [T]}$ and $\{R_{t^{\prime}_k, a_k}\}_{k \in [T]}$ are equal in distribution. Combining with \eqref{eq:this}, 
we have that 
$\{R^{\prime}_{t_k, a_k}\}_{k \in [T]}$ and $\{R^{\prime}_{t^{\prime}_k, a_k}\}_{k \in [T]}$ are equal in distribution. 
\end{proof}

\section{Proof of Theorem~\ref{theorem:iff}}
\label{section:iff}
We first establish Lemma~\ref{lemma:equivalence}, which provides a necessary and sufficient condition for two bandits to be equivalent; it is also key to the proof of Theorem~\ref{theorem:iff}. 
We present the lemma followed by its proof.  
\begin{lemma} [\bf{Necessary and Sufficient Condition for Equivalence}]
Let $\{R_t\}_{t \in \mathbb{N}}$ and $\{R^{\prime}_t\}_{t \in \mathbb{N}}$ be two bandits with the same finite set $\actions$ of actions. The bandits $\{R_t\}_{t \in \mathbb{N}} \sim \{R^{\prime}_t\}_{t \in \mathbb{N}}$ if and only if
for all $T \in \mathbb{N}$ and all sequences of actions $\{a_t\}_{t \in [T]}$ where each $a_t \in \actions$, the sequences of rewards $\{R_{t, a_t}\}_{t \in [T]}$ and $\{R^{\prime}_{t, a_t}\}_{t \in [T]}$ equal in distribution.
\label{definition:equivalence}
\label{lemma:equivalence}
\end{lemma}
\begin{proof}
Throughout this proof, we use $H_t^{\pi}$ to denote $H_t^{\pi \rightarrow R}$, $A_t$ to denote $A_t^{\pi \rightarrow R}$, $H_t^{\prime \pi}$ to denote $H_t^{\pi \rightarrow R^{\prime}}$, and $A_t^{\prime\pi}$ to denote $A_t^{\pi \rightarrow R^{\prime}}$. In addition, let $\mathcal{H}_t$ denote the set of histories of length $t$, i.e., the set of sequences of $t$ action-reward pairs. 

We prove each of the two directions below. 
\begin{enumerate}
\item \say{$\Rightarrow$}: Suppose that $\{R_t\}_{t \in \mathbb{N}} \sim \{R^{\prime}_t\}_{t \in \mathbb{N}}$.
    For all $T \in \mathbb{N}$, and all sequences of actions $\{a_t\}_{t \in [T]}$, we can construct a deterministic policy $\pi^{\{a_t\}_{t \in [T]}}$ that executes action $a_t$ at timestep $t$ for all $t \in \mathbb{N}_0$, $t \leq T$. Then $H_t^{\pi^{\{a_t\}_{t \in [T]}}}$ and $H_t^{\prime\pi^{\{a_t\}_{t \in [T]}}}$ have the same distribution. Hence,
    \begin{align*}
 \Pr\left(\{R_{t, a_t}\}_{t \in [T]} \in \cdot\right) 
 = \Pr\left(\{R^{\prime}_{t, a_t}\}_{t \in [T]} \in \cdot\right). 
    \end{align*}
\item \say{$\Leftarrow$}: Suppose that for all $T \in \mathbb{N}$ and all sequences of actions $\{a_t\}_{t \in [T]}$, where each $a_t \in \actions$, the sequence of rewards $\{R_{t, a_t}\}_{t \in [T]}$ and $\{R^{\prime}_{t, a_t}\}_{t \in [T]}$ have the same distribution. Let $\pi$ be any policy. We prove the hypothesis that
$H_T^{\pi}$ and $H_T^{\prime\pi}$
 have the same distribution for all $T \in \mathbb{N}_0$ by induction. 
    \begin{enumerate} [(a)]
    \item For $T = 0$, $H_0^{\pi}$ and $H_0^{\prime\pi}$ are both empty history, so they have the same distribution. 
    \item Now suppose that the hypothesis holds for $T = t$, i.e., $H_t^{\pi}$ and $H_t^{\prime \pi}$ have the same distribution. We prove for $t + 1$. 
    First, observe that for all $h_t \in \mathcal{H}_t$, 
    \begin{align}
        \Pr\left(A_t^{\pi} \in \cdot | H_t^{\pi} = h_t\right) = \pi\left(\cdot | h_t\right) = \Pr\left(A_t^{\prime\pi} \in \cdot | H_t^{\prime\pi} = h_t\right).
    \label{eq:equiv_1}
    \end{align}
    Next, for all $h_t \in \mathcal{H}_t$, $a_t \in \actions$, we have 
    \begin{align*}
        &\ \Pr\left(R_{t+1, A_t^{\pi}} \in \cdot | H_t^{\pi} = h_t, A_t^{\pi} = a_t\right) \\
        {=} &\ \Pr\left(R_{t+1, a_t} \in \cdot | H_t^{\pi} = h_t\right) \\
        \stackrel{}{=} &\ \Pr\left(R_{t+1, a_t} \in \cdot | R_{1, a_0} = r_{1, a_0}, R_{2, a_1} = r_{2, a_1}, ... , R_{t, a_{t-1}} = r_{t, a_{t-1}}\right) \\
       {=} &\ \Pr\left(R^{\prime}_{t+1, a_t} \in \cdot | R^{\prime}_{1, a_0} = r_{1, a_0}, R^{\prime}_{2, a_1} = r_{2, a_1}, ... , R^{\prime}_{t, a_{t-1}} = r_{t, a_{t-1}}\right) \\
        \stackrel{}{=} &\ \Pr\left(R^{\prime}_{t+1, a_t} \in \cdot | H_t^{\prime \pi} = h_t\right) \\
        \stackrel{}{=} &\ \Pr\left(R^{\prime}_{t+1, A_t^{\prime\pi}} \in \cdot | H_t^{\prime\pi} = h_t, A_t^{\prime\pi} = a_t\right), \numberthis
    \label{eq:equiv_2}
    \end{align*}
    where the first equality follows from $R_{t+1} \perp A_t^{\pi} | H_t^{\pi}$, the third equality holds because of the assumption that for all $T \in \mathbb{N}_0$ and all sequence of actions $\{a_t\}_{t \in [T]}$, where each $a_t \in \actions$, the sequence of rewards $\{R_{t, a_t}\}_{t \in [T]}$ and $\{R^{\prime}_{t, a_t}\}_{t \in [T]}$ have the same distribution, and the last equality follows from $R^{\prime}_{t+1} \perp A_t^{\prime\pi} | H_t^{\prime\pi}$. 
    
    Recall that $H_{t+1}^{\pi} = (H_t^{\pi}, A_t^{\pi}, R_{t+1, A_t^{\pi}})$ and $H_{t+1}^{\prime\pi} = (H_t^{\prime\pi}, A_t^{\prime\pi}, R^{\prime}_{t+1, A_t^{\prime\pi}})$. For all $h_{t+1} \in \mathcal{H}_{t+1}$, we let $h_t \in \mathcal{H}_t$, $a_t \in \actions$, and $r_{t+1, a_t} \in \mathbb{R}$ be defined such that $h_{t+1} = (h_t, a_t, r_{t+1, a_t})$. Therefore, for all $t \in \mathbb{N}_0$ and $h_{t+1} \in \mathcal{H}_{t+1}$, 
    \begin{align*}
        &\ \Pr(H_{t+1}^{\pi} = h_{t+1})\\
        = &\ \Pr(H_t^{\pi} = h_t) \Pr(A_t^{\pi} = a_t | H_t^{\pi} = h_t) \Pr(R_{t+1, A_t^{\pi}} = r_{t+1, a_t} | H_t^{\pi} = h_t, A_t^{\pi} = a_t)\\
        \stackrel{}{=} &\ \Pr\left(H_t^{\prime\pi} = h_t\right) \Pr\left(A_t^{\prime\pi} = a_t | H_t^{\prime\pi} = h_t\right) \Pr\left(R^{\prime}_{t+1, A_t^{\prime\pi}} = r_{t+1, a_t} | H_t^{\prime\pi} = h_t, A_t^{\prime\pi} = a_t\right)\\
        = &\ \Pr\left(H_{t+1}^{\prime \pi} = h_{t+1}\right), 
    \end{align*}
    where the second equality follows from our hypothesis for $T = t$, \eqref{eq:equiv_1} and \eqref{eq:equiv_2}. 
    \end{enumerate}
    Hence, we've proved by induction that $H_T^{\pi}$ and $H_T^{\prime \pi}$ have the same distribution for all policies $\pi$ and all $T \in \mathbb{N}_0$. 
\end{enumerate}
\end{proof}

We are now ready to prove Theorem~\ref{theorem:iff}. 
\iffstationarity*
\begin{proof}
We prove the two directions, respectively. 
\begin{enumerate}
\item \say{$\Leftarrow$}: Suppose there exists a bandit environment $\{R_t^{\prime}\}_{t \in \mathbb{N}}$ where the sequence of rewards is exchangeable such that $\{R_t\}_{t \in \mathbb{N}} \sim \{R_t^{\prime}\}_{t \in \mathbb{N}}$. Then for all $T \in \mathbb{N}$, any sequence of actions $\{a_k\}_{k \in [T]}$ in the shared action set, and any sequences of distinct timesteps $\{t_k\}_{k \in [T]}$ and $\{t^{\prime}_k\}_{k \in [T]}$, 
\begin{align*}
   &\ \Pr\left(\{R_{t_k, a_k}\}_{k \in [T]} \in \cdot\right)
    = \Pr\left(\left\{R^{\prime}_{t_k, a_k}\right\}_{k \in [T]} \in \cdot\right) \\
    = &\ \Pr\left(\left\{R^{\prime}_{t^{\prime}_k, a_k}\right\}_{k \in [T]} \in \cdot\right) 
    = \Pr\left(\left\{R_{t^{\prime}_k, a_k}\right\}_{k \in [T]} \in \cdot\right),
\end{align*}
where the first and last equalities follow from the equivalence of the two bandit environments, and the second equality follows from the fact that the sequence of rewards in $\{R^{\prime}_t\}_{t \in \mathbb{N}}$ is exchangeable. Thus, we have shown that $\{R_{t}\}_{t \in \mathbb{N}}$ is stationary. 
\item \say{$\Rightarrow$}: Now suppose that $\{R_{t}\}_{t \in \mathbb{N}}$ is stationary. As an outline of the proof, we construct a bandit $\{R^{\prime}_{t}\}_{t \in \mathbb{N}}$ where the sequence of rewards is exchangeable and show that $\{R_{t}\}_{t \in \mathbb{N}} \sim \{R^{\prime}_{t}\}_{t \in \mathbb{N}}$. 

We first look at the sequence of $|\actions|-$tuples $\{(R_{n|\actions|+1, 1}, \cdots , R_{n|\actions| + |\actions|, |\actions|})\}_{n \in \mathbb{N}_0}$. Since the bandit $\{R_t\}_{t \in \mathbb{N}}$ is stationary,  the sequence of $|\actions|$-tuples is exchangeable. By de Finetti's theorem, there exists a random distribution $P$ over $\R^{|\actions|}$ such that conditioned on $P$, 
\begin{align*}
    \left\{\left(R_{n|\actions|+1, 1}, \cdots , R_{n|\actions| + |\actions|, |\actions|}\right)\right\}_{n \in \mathbb{N}_0} \stackrel{\mathrm{i.i.d.}}{\sim} P,  
\end{align*}
In other words, conditioned on $P$, the $|\actions|$-tuples are independently and identically distributed according to $P$. 

We next construct a bandit $\{R^{\prime}_{t}\}_{t \in \mathbb{N}}$ with a set $\actions$ of actions, where the rewards are independently and identically distributed according to $P$, conditioned on $P$. 
We also require that conditioned on $P$, the reward sequence in $\{R^{\prime}_t\}_{t \in \mathbb{N}}$ is independent of the reward sequence in $\{R_t\}_{t \in \mathbb{N}}$. 
Observe that the sequence of rewards in the bandit $\{R_t^{\prime}\}_{t \in \mathbb{N}}$ is exchangeable. 

Below we show that the two bandit environments $\{R_{t}\}_{t \in \mathbb{N}}$ and $\{R^{\prime}_{t}\}_{t \in \mathbb{N}}$ are equivalent. 
According to Lemma~\ref{lemma:equivalence},  
we need to show that for all $T \in \mathbb{N}$ and all sequences of actions $\{a_t\}_{t \in [T]}$, the sequences of rewards $\{R_{t, a_t}\}_{t \in [T]}$ and $\{R^{\prime}_{t, a_t}\}_{t \in [T]}$ have the same distribution.  
With out loss of generality, let $\actions = \{1, 2, ... , |\actions|\}$. 
We define $S_k = \{t \in [T]: a_t = k\}$ for all $k \in \actions$. 
Then $S_k$ represents the set of timesteps that action $k$ appears in the sequence $\{a_t\}_{t \in [T]}$. 
It suffices to show that 
\begin{align}
     \Pr\left(\left(R_{S_1, 1}, R_{S_2, 2}, \cdots , R_{S_{|\actions|}, |\actions|}\right) \in \cdot\right)
    = \Pr\left(\left(R^{\prime}_{S_1, 1}, R^{\prime}_{S_2, 2}, \cdots, R^{\prime}_{S_{|\actions|},|\actions|}\right) \in \cdot\right).
\label{eq:equiv_envs}
\end{align}
Let $C_k = \sum_{i = 1}^k |S_i|$ for all $k \in \actions$. Then $C_k$ represents the number of timesteps that actions $1$ through $k$ appear in the sequence $\{a_t\}_{t \in [T]}$. Then we have
\begin{align*}
       &\ \Pr\left(\left(R_{S_1, 1}, R_{S_2, 2}, \cdots , R_{S_{|\actions|}, |\actions|}\right) \in \cdot\right)\\
       \stackrel{}{=} &\ \Pr\left(\left(\{R_{n|\actions| + 1, 1}\}_{n =1}^{C_1}, \{R_{n|\actions| + 2, 2}\}_{n = C_1 + 1}^{C_2}, \cdots, \{R_{n|\actions| + |\actions|, |\actions|}\}_{n = C_{|\actions| - 1}+1}^{C_{|\actions|}} \right)\in \cdot \right)\\
       \stackrel{}{=} &\ \Pr\left(\left(R^{\prime}_{1 : C_1, 1}, R^{\prime}_{C_1 + 1: C_2, 2}, \cdots , R^{\prime}_{C_{|\actions| - 1} + 1: C_{|\actions|},|\actions|} \right)\in \cdot\right)\\
       \stackrel{}{=} &\ \Pr\left(\left(R^{\prime}_{S_1, 1}, R^{\prime}_{S_2, 2}, \cdots, R^{\prime}_{S_{|\actions|},|\actions|}\right) \in \cdot\right), 
\end{align*}
where the first equality follows from the stationarity of $\{R_t\}_{t \in \mathbb{N}}$, 
the second equality follows from the construction of $\{R^{\prime}_t\}_{t \in \mathbb{N}}$, 
and the last equality from that the sequence of rewards associated with $\{R_t^{\prime}\}_{t \in \mathbb{N}}$ is exchangeable. Hence, we've shown that \eqref{eq:equiv_envs} holds.
\end{enumerate}
\end{proof}

\section{Proof of Theorem~\ref{theorem:strongly_stationary}}
\label{section:strongly}
\landscape*
\begin{proof} We prove each of the two statements below. 
\begin{itemize}
\item We prove the two directions, respectively.  
\begin{enumerate} 
    \item \say{$\Leftarrow$}: First observe that in a strongly stationary bandit, the sequence of rewards is exchangeable. Then by Theorem~\ref{theorem:iff}, we prove this direction. 
    \item \say{$\Rightarrow$}: Suppose that $\{R_t\}_{t \in \mathbb{N}}$ is stationary. We construct a strongly stationary bandit $\{R^{\prime}_t\}_{t \in \mathbb{N}}$ such that  $\{R_t\}_{t \in \mathbb{N}} \sim \{R^{\prime}_t\}_{t \in \mathbb{N}}$. Specifically, we construct $\{R^{\prime}_t\}_{t \in \mathbb{N}}$ in the same way as in the proof of Theorem~\ref{theorem:iff}. Now we show that the bandit $\{R^{\prime}_t\}_{t \in \mathbb{N}}$ is strongly stationary. For all $T \in \mathbb{N}$, 
    \begin{align*}
        \Pr\left(\{R^{\prime}_{t}\}_{t \in [T]} \in \cdot\right) 
        = &\ \Pr\left(\{(R_{n|\actions|+1, 1}, ... , R_{n|\actions| + |\actions|, |\actions|})\}_{n \in [T]} \in \cdot\right) \\
        = &\ 
        \Pr\left(\{(R^{\prime}_{n|\actions|+1, 1}, ... , R^{\prime}_{n|\actions| + |\actions|, |\actions|})\}_{n \in [T]} \in \cdot\right), 
    \end{align*}
    where the first equality follows from the construction of $\{R^{\prime}_t\}_{t \in [T]}$, and the second equality follows from the equivalence of the two bandit environments.
    \end{enumerate}
\item By Lemma~\ref{lemma:equivalence}, the direction \say{$\Leftarrow$} is straightforward. Now we prove the other direction \say{$\Rightarrow$}. 
    Suppose two strongly stationary bandits $\{R_t\}_{t \in \mathbb{N}}$ and $\{R^{\prime}_t\}_{t \in \mathbb{N}}$ are equivalent.
    Then for all $T \in \mathbb{N}_0$, 
    \begin{align*}
        \Pr\left(\{R_{t}\}_{t \in [T]} \in \cdot\right) 
        = &\ \Pr\left(\{(R_{n|\actions|+1, 1}, ... , R_{n|\actions| + |\actions|, |\actions|})\}_{n \in [T]} \in \cdot\right)\\
        = &\  
        \Pr\left(\{(R^{\prime}_{n|\actions|+1, 1}, ... , R^{\prime}_{n|\actions| + |\actions|, |\actions|})\}_{n \in [T]} \in \cdot\right)\\
        = &\ \Pr\left(\{R^{\prime}_{t}\}_{t \in [T]} \in \cdot\right), 
    \end{align*}
where the first and the third equalities follow from the fact that the two bandits are strongly stationary, and the second inequality follows from equivalence. 
\end{itemize}
\end{proof}

\section{Regret Analysis}
\label{section:analysis}
Recall that by taking $\chi_t = R_{1:t}$ or $\chi_t = R_{t+2:+\infty}$, we retrieve the two definitions of regret introduced by \citep{liu2022nonstationary}, and that PS is proven to be consistent with its stationary counterpart, TS. The regret analysis proposed by \citep{liu2022nonstationary} suggests that we can utilize notions of regret that are consistent with the conventional notion to derive regret analysis to evaluate the performance of agents, including one that is consistent with its stationary counterpart. Here we present a direct generalization of that regret analysis. 

\begin{theorem} [\bf{General Regret Bound}]
\label{theorem:general_regret} 
For all deterministic function sequences $\{f_t\}_{t\in \mathbb{N}}$, stochastic processes $\chi = \{\chi_t\}_{t \in \mathbb{N}}$, policies $\pi$, and $T \in \mathbb{N}$, if $\alpha_t = f_t(R_t)$, then we have 
\begin{align*}
    {\mathrm{Regret}}^{\chi}(T; \pi) \leq \sqrt{\sum_{t = 0}^{T-1} \Gamma_t^{\pi} \left(\chi; \alpha\right) V_T\left(\alpha\right)}, 
\end{align*}
where $\Gamma_t^{\pi}(\chi; \alpha) =  \frac{\mathbb{E}[\max_{a \in \actions}\mathbb{E}[R_{t+1,a}|\chi_t] - R_{t+1, A_t^{\pi}}]^2}{\I\left(\alpha_{t+2:\infty}; A_t^{\pi}, R_{t+1, A_t^{\pi}} | H_t^{\pi}\right)}$
and $V_T(\alpha) = \sum_{t = 0}^{T-1} \I(\alpha_{t+2:\infty}; R_{t+1} | R_{1:t}).$ 
\label{theorem:general_regret}
\end{theorem}
\begin{proof}
Let $\alpha = \{\alpha_t\}_{t \in \mathbb{N}}$ be a stochastic process satisfying the criterion. Then for all stochastic processes $\chi$,  
policies $\pi$ and $T \in \mathbb{N}$, 
\begin{align*}
 \mathrm{Regret}^{\chi}(T; \pi) 
= &\ \sum_{t= 0 }^{T-1} \mathrm{Regret}^{\chi}_{\mathrm{ins}}(t;\pi) \\
 \overset{}{\leq} &\ \sum_{t=0}^{T-1} 
\sqrt{\Gamma_{t}^{\pi}\left(\chi; \alpha\right) \I\left(\alpha_{t+2:\infty}; A_t^{\pi}, R_{t+1, A_t^{\pi}}|H_t^{\pi}\right)}\\
 \overset{}{\leq} &\ \sqrt{\sum_{t=0}^{T-1} \I\left(\alpha_{t+2:\infty}; A_t^{\pi}, R_{t+1, A_t^{\pi}}|H_t^{\pi}\right)} \sqrt{\sum_{t = 0}^{T-1}{\Gamma}^{\pi}_t\left(\chi; \alpha\right)}, 
\numberthis
\label{eq:main_proof_eq_1}
\end{align*}
where the first inequality follows from the definition of the information ratio, 
and the last inequality follows from the Cauchy-Schwartz inequality.
Observe that 
 $\alpha_{t+2:\infty} \perp H^{\pi}_{t+1} | R_{1:t+1}$. 
Hence, for all stochastic processes $\chi$,  policies $\pi$ and $t \in \mathbb{N}$, 
\begin{align*}
       \I\left(\alpha_{t+2:\infty}; A_t^{\pi}, R_{t+1, A_t^{\pi}} | H_t^{\pi}\right) 
    = \I\left(\alpha_{t+2:\infty}; R_{1:t+1} | H_t^{\pi}\right) 
    - \I\left(\alpha_{t+2:\infty}; R_{1:t+1}|H_{t+1}^{\pi}\right).
\end{align*}
Therefore, for all stochastic processes $\chi$, policies $\pi$ and $T \in \mathbb{N}$, 
\begin{align*}
     &\ \sum_{t = 0}^{T-1} \I\left(\alpha_{t+2:\infty}; A_t^{\pi}, R_{t+1, A_t^{\pi}} | H_t^{\pi}\right)\\
    \stackrel{}{=} &\ \sum_{t = 0}^{T-1}\left[\I\left(\alpha_{t+2:\infty}; R_{1:t+1} | H_t^{\pi}\right) - \I\left(\alpha_{t+2:\infty}; R_{1:t+1}|H_{t+1}^{\pi}\right) \right]\\
     \leq &\ \I(\alpha_{2:\infty}; R_1)  + \sum_{t = 1}^{T-1} \left[\I\left(\alpha_{t+2:\infty}; R_{1:t+1} | H_t^{\pi}\right) - \I\left(\alpha_{t+1:\infty}; R_{1:t}|H_t^{\pi}\right)\right] \\ 
     \stackrel{}{=} &\ \I(\alpha_{2:\infty}; R_1) 
    + 
    \sum_{t = 1}^{T-1} \left[\I\left(\alpha_{t+2:\infty}; R_{1:t} | H_t^{\pi}\right) +
   \I\left(\alpha_{t+2:\infty}; R_{t+1} | R_{1:t}, H_t^{\pi}\right)
    - \I\left(\alpha_{t+1:\infty}; R_{1:t}|H_t^{\pi}\right)\right] \\ 
     \leq &\ \I(\alpha_{2:\infty}; R_1) + 
    \sum_{t = 1}^{T-1} \I\left(\alpha_{t+2:\infty}; R_{t+1} | R_{1:t}, H_t^{\pi}\right) \\
     \stackrel{}{=} &\ \I(\alpha_{2:\infty}; R_1) + 
    \sum_{t = 1}^{T-1} 
    \I\left(\alpha_{t+2:\infty}; R_{t+1} | R_{1:t}\right) \\
    = &\ V_T(\alpha), 
    \numberthis
    \label{eq:main_proof_eq_2}
\end{align*}
where the second equality follows from the chain rule of mutual information, and the second-to-last equality follows from 
$(\alpha_{t+2:\infty}, R_{t+1}) \perp H_t^{\pi}|R_{1:t}$. 
Incorporating \eqref{eq:main_proof_eq_1} and \eqref{eq:main_proof_eq_2}, we complete the proof. 
\end{proof}

In Theorem~\ref{theorem:general_regret}, $V_T$ is the \emph{predictive information} \citep{liu2022nonstationary}, representing information that is useful in predicting future $\alpha_t$'s; $\Gamma_t^{\pi}$ is the \emph{information ratio} \citep{bubeck2015bandit, dong2018information, hao2022contextual, lattimore2019information, lu2023reinforcement, neu2022lifting, RussoMOR2014, russo2016information, russo2018learning}, representing how (in)effcient an agent $\pi$ is at trading-off information acquisition and reward maximization. 

Below we provide a general upper bound on predictive information. 
\begin{proposition} [\bf{General Upper Bound on Predictive Information}]
\label{proposition:markov_info}
Let $\{S_t\}_{t \in \mathbb{N}}$ 
be a Markov process 
such that, for all $t \in \mathbb{N}_0$, 
 $S_{t+1:\infty} \perp R_{1:t} | S_{t}$ and 
 $\alpha_{t+2:\infty} \perp R_{t+1} | S_{t+2}, R_{1:t}$. 
For all $T \in \mathbb{N}$, the cumulative predictive information is defined as $V_T(\alpha) = \sum_{t = 0}^{T-1} \I(\alpha_{t+2:\infty}; R_{t+1} | R_{1:t})$, and 
satisfies 
\begin{align*}
V_T(\alpha)  \leq \I(S_2; S_1) + 
 \sum_{t = 1}^{T-1} \I(S_{t+2}; S_{t+1} | S_{t}).  
\end{align*}
\end{proposition}
\vspace{-3mm}
\begin{proof}
Observe that for all $t \in \mathbb{N}_0$, the amount of incremental predictive information satisfies 
\begin{align*}
\Delta_t(\alpha) = &\ \I(\alpha_{t+2:\infty}; R_{t+1} | R_{1:t})\\
\stackrel{}{\leq} &\ \I(S_{t+2}; R_{t+1} | R_{1:t}) \\
= &\ \H(S_{t+2} | R_{1:t}) - \H(S_{t+2} | R_{1:t+1})\\
= &\ \H(S_{t+2} | R_{1:t}, S_{t+1}) 
+ \I(S_{t+2}; S_{t+1} | R_{1:t}) 
- \H(S_{t+2} | R_{1:t+1}, S_{t+1})
- \I(S_{t+2}; S_{t+1} | R_{1:t+1})\\
\stackrel{}{=} &\ \H(S_{t+2} | S_{t+1}) 
+ \I(S_{t+2}; S_{t+1} | R_{1:t}) 
- \H(S_{t+2} | S_{t+1})
- \I(S_{t+2}; S_{t+1} | R_{1:t+1})\\
= &\ \I(S_{t+2}; S_{t+1} | R_{1:t}) 
- \I(S_{t+2}; S_{t+1} | R_{1:t+1}), 
\end{align*}
where the first inequality follows from $\alpha_{t+2:\infty} \perp R_{t+1} | S_{t+2}, R_{1:t}$ and data-processing inequality, and the second-to-last equality follows from $S_{t+2} \perp R_{1:t+1} | S_{t+1}$. 

Then for all $T \in \mathbb{N}$, the amount of cumulative predictive information can be upper-bounded:
\begin{align*}
V_T(\alpha) = 
\sum_{t = 0}^{T-1} \Delta_t (\alpha) = &\ 
\sum_{t = 0}^{T-1} \left[\I(S_{t+2}; S_{t+1} | R_{1:t}) 
- \I(S_{t+2}; S_{t+1} | R_{1:t+1})\right]\\
= &\ \I(S_2; S_1) + 
\sum_{t = 1}^{T-1} \left[\I(S_{t+2}; S_{t+1} | R_{1:t}) 
- \I(S_{t+1}; S_{t} | R_{1:t})\right] \\
\leq &\ \I(S_2; S_1) + 
\sum_{t = 1}^{T-1} \left[\I(S_{t+2}, S_{t}; S_{t+1} | R_{1:t}) 
- \I(S_{t+1}; S_{t} | R_{1:t})\right] \\
= &\ \I(S_2; S_1) + 
\sum_{t = 1}^{T-1} \left[\I(S_{t+2}; S_{t+1} | R_{1:t}, S_{t}) \right] \\
\stackrel{(a)}{=} &\ \I(S_2; S_1) + 
\sum_{t = 1}^{T-1} \left[\I(S_{t+2}; S_{t+1} | S_{t}) \right], 
\end{align*}
where the last equality follows from $(S_{t+1}, S_{t+2}) \perp R_{1:t} | S_{t}$. 
\end{proof}

The random variable $S_t$ can be thought of as the hidden state of the bandit at timestep $t$. By suitably constructing $\{S_t\}_{t \in \mathbb{N}}$, we can apply this lemma to bound the amount of predictive information in any bandit. In particular, 
this can be achieved by letting $S_t = R_{1:t}$ for all $t \in \mathbb{N}$.  
For some bandits, there is a better choice of $\{S_t\}_{t \in \mathbb{N}}$ with which we can derive sharper bounds or more interpretable bounds when applying the lemma. For example, in a modulated Bernoulli bandit or an AR(1) bandit, we can conveniently let $S_t = \theta_t$, if $\alpha_t = f_t(R_t)$ and $\{f_t\}_{t \in \mathbb{N}}$ is a sequence of deterministic functions. 


To relate our regret bound to existing regret bounds established in the literature of stationary bandits, we compare our bound to a bound established by \cite{neu2022lifting}. 
Observe that in a stationary bandit, if the information ratio satisfies $\Gamma_t^{\pi}(\chi; \alpha) \leq \overline{\Gamma}^{\pi}(\chi; \alpha)$ for all $t \in \mathbb{N}$ 
for some $\overline{\Gamma}^{\pi}(\chi; \alpha)$, 
then our result establishes that 
${\mathrm{Regret}}^{\chi}(T;\pi) \leq \sqrt{\overline{\Gamma}^{\pi}(\chi; \alpha)T\H(\theta)}$, by letting $S_t$ be an invariant latent parameter $\theta$ that determines the reward distribution. 
If we in addition let each $\chi_t$ be the reward distribution and $\alpha_t = R_t$,  
then our regret bound is equivalent to an information-theoretic regret bound established by \cite{neu2022lifting}.

\medskip









\end{document}